\documentclass[twoside]{article}

\usepackage[accepted]{aistats2021}
\usepackage[utf8]{inputenc} % allow utf-8 input
\usepackage[T1]{fontenc}    % use 8-bit T1 fonts
\usepackage[hidelinks]{hyperref}       % hyperlinks
\usepackage{url}            % simple URL typesetting
\usepackage{booktabs}       % professional-quality tables
\usepackage{amsfonts}       % blackboard math symbols
\usepackage{nicefrac}       % compact symbols for 1/2, etc.
\usepackage{microtype}      % microtypography

\usepackage{graphicx}
\usepackage[numbers]{natbib}
\usepackage{amsmath}
\usepackage{algorithm}
\usepackage[noend]{algpseudocode}
\usepackage{wrapfig}
\usepackage{amsthm}
\usepackage{amssymb}
\usepackage{color}
\usepackage{xcolor}
\usepackage{array}
\usepackage{float}
\usepackage{bbm}
\usepackage{multirow}
\usepackage{longtable}

\newtheorem{lemma}{Lemma}
\newtheorem{theorem}{Theorem}

\newtheorem{definition}{Definition}
\newtheorem{assumption}{Assumption}
\newtheorem{prop}{Proposition}

\theoremstyle{remark}
\newtheorem*{remark}{Remark}

\theoremstyle{definition}

\DeclareMathOperator*{\argmax}{argmax}
\DeclareMathOperator*{\argmin}{argmin}

\newcommand{\Ct}{C_{T,\delta}}

\begin{document}

\twocolumn[

\aistatstitle{Learning the Truth From Only One Side of the Story}

\aistatsauthor{ Heinrich Jiang$^*$ \And Qijia Jiang$^*$ \And Aldo Pacchiano$^*$ }

\aistatsaddress{ Google Research \\ \texttt{heinrichj@google.com} \And Stanford University \\ \texttt{qjiang2@stanford.edu} \And UC Berkeley \\ \texttt{pacchiano@berkeley.edu} } ]

\begin{abstract}
Learning under one-sided feedback (i.e., where we only observe the labels for examples we predicted positively on) is a fundamental problem in machine learning -- applications include lending and recommendation systems. Despite this, there has been surprisingly little progress made in ways to mitigate the effects of the sampling bias that arises. We focus on generalized linear models and show that without adjusting for this sampling bias, the model may converge suboptimally or even fail to converge to the optimal solution. We propose an adaptive approach that comes with theoretical guarantees and show that it outperforms several existing methods empirically. Our method leverages variance estimation techniques to efficiently learn under uncertainty, offering a more principled alternative compared to existing approaches.
\end{abstract}

% Old Abstract
% Learning under one-sided feedback (i.e., where examples arrive in an online fashion and the learner only sees the labels for examples it predicted positively on) is a fundamental problem in machine learning -- applications include lending and recommendation systems. Despite this, there has been surprisingly little progress made in ways to mitigate the effects of the sampling bias that arises.
%We focus on generalized linear models and show that without adjusting for this sampling bias, the model may converge sub-optimally or even fail to converge to the optimal solution. We propose an adaptive Upper Confidence Bound approach that comes with rigorous regret guarantees and we show that it outperforms several existing methods experimentally. Our method leverages uncertainty estimation techniques for generalized linear models to more efficiently explore uncertain areas than existing approaches which explore randomly. 
\section{INTRODUCTION}
\label{sec::intro}

Machine learning is deployed in a wide range of critical scenarios where the feedback is one-sided, including bank lending \cite{tsai2010credit,kou2014mcdm,tiwari2018machine}, criminal recidivism prediction \cite{tollenaar2013method,wang2010predicting,berk2017impact}, credit card fraud \cite{chan1999distributed,srivastava2008credit}, spam detection \cite{jindal2007review,sculley2007practical}, self-driving motion planning \cite{paden2016survey,lee2014local}, and recommendation systems \cite{pazzani2007content,covington2016deep,he2014practical}. These applications can often times be modeled as one-sided feedback in that the true labels are only observed for the positively predicted examples and the learner is simultaneously making predictions and actively learning a better model. For example, in bank loans, the learner only observes whether the loan was repaid if it was approved. In criminal recidivism prediction, the decision maker only observes any re-offences for inmates who were released.% Finally, large internet companies have recommendation systems (i.e., recommending ads, videos, posts, news, etc.) that are critical to user satisfaction and the success of the service. In such systems, clicks are observed only for content actually shown to the user.

Incidentally, this problem can be viewed as a variation on the classical active learning problem in the streaming setting \cite{bordes2005fast,chu2011unbiased,lu2016online}, where unlabeled examples arrive in a sequential manner and the learner must decide whether to query for its label for a fixed cost in order to build a better model. Here, the goal is similar, with the difference that the labels being queried are the ones with positive predictions. There is a tension between making the correct predictions and choosing the right examples to query for labels -- a cost is associated with querying negative examples on one hand, and on the other we seek to learn a better model for improved future performance. As we show later, the key difficulty of this problem lies in understanding this trade-off and exactly pinpointing when to make a positive prediction in the face of uncertainty. In the case of bank lending, for example, assessing the confidence for the prediction on applicant's chance of repayment is of great importance. Decision needs to be made on balancing the risk of default if granted the loan, which comes with a high cost, and the benefit of the additionally gathered data our model can learn from.
%if we are confident about our model's prediction of whether or not an applicant will repay the loan, the decision is straightforward; however, in the cases where our model is uncertain, there is a critical trade-off: if we 

One often overlooked aspect is that the samples used to train the model, which prescribes which data points we should act upon next, are inherently biased by its own past predictions. In practical applications, there is a common belief that the main issue caused by such one-sided sampling is label imbalance \cite{he2014practical}, as the number of positive examples will be expected to be much higher than overall for the population. Indeed, this biasing of the labels leading to label imbalance can be a challenge, motivating much of the vast literature on label imbalance. However, the challenges go beyond label imbalance. We show that without accounting for such potential myopia caused by biased sampling, it is possible that we under-sample in regions where the model makes false negative predictions, and even with continual feedback, the model never ends up correcting itself. In the bank loan example, such under-sampling may systematically put minority group in a disadvantaged position, as reflected by the error being disproportionally attributed across groups, if we content ourselves with a point estimator that doesn't take into consideration the error bar that's associated.

% Moreover, such under-sampling may potentially disadvantage certain subpopulations disporportionally e.g. in bank lending, minority groups may end up systemically denied loans if the the model is initially biased against them.

% \qijia{This doesn't belong here.}
% Despite the importance of learning in this one-sided feedback setting, there has been surprisingly little work done in studying the effects of such biased sampling and how to mitigate it. Learning with partial feedback was first studied by \citet{helmbold2000apple} under the name ``apple tasting" who suggest to transform any learning procedure into an apple tasting one by randomly flipping some of the negative predictions into positive ones with probability decaying over time. They show upper and lower bounds on the number of mistakes made by the procedure in the partial feedback setting. Since then, there has only been a handful of works on this challenging yet ubiquitous problem in machine learning, which we outline in the related works section. 

In this paper, we take a data-driven approach to guide intervention efforts on correcting for the bias -- uncertainty quantification tools are used for striking the balance between short-term desideratum (i.e., low error rate on current sample) and long-term welfare (i.e., information collection for designing optimal policy). More concretely, we focus on generalized linear models, borrowing assumptions from a popular framework of \citet{glm_bandit}. Our contributions can be summarized as follows.
\begin{itemize}
    \item In Section~\ref{sec:setup}, we propose an objective, {\it one-sided loss}, to capture the one-sided learner's goals for the model under consideration.
    \item In Section~\ref{sec:slower_convergence}, we show that without leveraging active learning where the model is continuously updated upon seeing new labeled examples, a model may need to be trained on a sub-optimal amount of data to achieve a desirable performance on the objective. %as many as $\widetilde{\mathcal{O}}(1/\epsilon^3)$ samples to attain an average loss of at most $\epsilon$. 
    \item In Section~\ref{sec:greedy_lower}, we show that the greedy active approach (i.e., updating the model only on examples with positive predictions at each timestep) in general will not exhibit asymptotically vanishing loss.
    \item In Section~\ref{sec:adaptive}, we give a strategy that adaptively adjusts the model decision by incorporating the uncertainty of the prediction and show an improved rate of convergence on the objective. % $\widetilde{\mathcal{O}}(T^{1/2})$ regret.
    \item In Section~\ref{sec:sgd}, we explore the option of using iterative methods for learning the optimal model parameters while maintaining small misclassification rate under this partial feedback setting. The proposed SGD variant of the adaptive method complements our main results which focus on models fully optimized on all of the labeled examples observed so far.
    \item In Section~\ref{sec:experiments}, we provide an extensive experimental analysis on linear and logistic regression on various benchmark datasets showing that our method outperforms a number of baselines widely used in practice.
\end{itemize}
To the best of our knowledge, we give the most detailed analysis in the ways in which passive or greedy learners are sub-optimal in the one-sided feedback setting and we present a practical algorithm that comes with rigorous theoretical guarantees which outperforms existing methods empirically. 

% Our method is {\it adaptive}: it leverages variance estimation techniques for generalized linear models to more efficiently label query uncertain areas than existing approaches which perform the querying randomly. 
%Such efficient exploration is critical: in practice, mistakes can be costly (e.g. when a bank gives a defaulting loan) and at the same time we show principled exploration is necessary for the benefit in the long run (e.g. the bank needs to take a chance on some loan applicants in order to find more profitable lending opportunities in the future). 

\section{RELATED WORK}
\label{sec:related_works}
% As mentioned in the introduction, this problem of active learning with one-sided feedback has been studied by \citet{helmbold2000apple} under the name ``apple tasting". They propose randomly flipping some of the negative predictions to positive ones with probability decaying over time. Therefore their method can be seen as performing the exploration {\it randomly} while we take a more adaptive approach.
\vspace{-0.1cm}
Despite the importance and ubiquity of this active learning problem with one-sided feedback, there has been surprisingly little work done in studying the effects of such biased sampling and how to mitigate it. Learning with partial feedback was first studied by \citet{helmbold2000apple} under the name ``apple tasting" who suggest to transform any learning procedure into an apple tasting one by randomly flipping some of the negative predictions into positive ones with probability decaying over time. They give upper and lower bounds on the number of mistakes made by the procedure in this setting. 
% Since then, there has only been a handful of works on this challenging yet ubiquitous problem in machine learning, which we outline in the related works section. 
\citet{sculley2007practical} studies the one-sided feedback setting for the application of email spam filtering and show that the approach of \citet{helmbold2000apple} was less effective than a simple greedy strategy. % and explored active learning approaches to solve the problem
\citet{cesa2006worst} propose an active learning method for linear models to query an example's label randomly with probability based on the model's prediction score for that example. \citet{bechavod2019equal} consider the problem of one-sided learning in the group-based fairness context with the goal of satisfying equal opportunity \cite{hardt2016equality} at every round. They consider convex combinations over a finite set of classifiers and arrive at a solution which is a randomized mixture of at most two of these classifiers. 

\citet{cesa2006regret} studies a setting which generalizes the one-sided feedback, called {\it partial monitoring}, through considering repeated two-player games in which the player receives a feedback generated by the combined choice of the player and the environment. They propose a randomized solution.  \citet{antos2013toward} provides a classification of such two-player games in terms of the regret rates attained and \citet{bartok2012partial} study a variant of the problem with side information. Our approach does not rely on randomization that is typically required to solve such two-player games. There has also been work studying the effects of distributional shift caused by biased sampling  \cite{perdomo2020performative}. \citet{ensign2017decision} studies the one-sided feedback setting through the  problems of predictive policing and recidivism prediction. They show a reduction to the partial monitoring setting and provide corresponding regret guarantees.

%More broadly, as discussed earlier, this one-sided learning problem is related to selective sampling or active learning where the learner chooses which labels to observe.%The difference in our setting is that we incur a cost when we query for a label that's negative and the goal is to query exactly the positively labeled examples.

\citet{glm_bandit} propose a generalized linear model framework for the multi-armed bandit problem, where for arm $a$, the reward is of the form $\mu(a^{\top} \beta^*) + \epsilon$ where $\beta^*$ is unknown to the learner, $\epsilon$ is additive noise, and $\mu(\cdot)$ is a link function. Our work borrows ideas from this framework as well as proof techniques. 
Their notion of regret is based on the difference between the expected reward of the chosen arm and that of an optimal arm. One of our core contributions is showing that, surprisingly, modifications to the GLM-UCB algorithm leads to a procedure that minimizes a very different objective under a disparate feedback model. % that is one-sided and not compared to any single arm but to the best context-dependent decision at each time-step. 

\section{PROBLEM SETUP}
\label{sec:setup}

We assume that data pairs $(x,y) \in \mathbb{R}^d\times \mathbb{R}$ are streaming in and the learner interacts with the data in sequential rounds: at time step $t$ we are presented with a batch of $N$ samples $(x_1^t, ,\cdots,x_N^t)$, and for the data points we decide to observe, we are further shown the corresponding labels $y_i^t$, while no feedback is provided for the unobserved ones. We make the following assumptions.
\begin{assumption}[GLM Model]
\label{assumption:setup} 
There exists $\beta^* \in \mathbb{R}^d$ (unknown to the learner) and link function $\mu : \mathbb{R} \mapsto \mathbb{R}$ (known to the learner) such that
$y$ is drawn according to an additive noise model $y=\mu(x^{\top}\beta^*)+\epsilon$. The link function $\mu(\cdot)$ is continuously differentiable and strictly monotonically increasing, with Lipschitz constant $L$, i.e., $ 0 < \mu'(z) \le  L \;\forall z \in \mathbb{R}$. Moreover, $\mu(0)\leq \gamma$. %\qijia{Where did we use the last one?} 
\end{assumption}
\begin{assumption}[Bounded Covariate] There exists some $B > 0$ such that $\|x_i^t\|_2\leq B$ for all $i \in [N], t\ge 0$. %, $|y_i^t|\leq C$ 
\end{assumption}
\begin{assumption}[Parameter Diameter]
The unknown parameter $\beta^*$ satisfies $\|\beta^*\|_2 \leq M$.
\end{assumption}
\begin{assumption}[Subgaussian Noise]
The noise residuals $\epsilon_i^t := y_i^t - \mu(x_i^{t\top}\beta^*)$ are  mutually independent, conditionally zero-mean and conditionally $\phi$-subgaussian. That is, $\forall i \in [N], t \ge 1, \tau\in\mathbb{R}$,
\[\mathbb{E}[\epsilon_i^t|\{x_i^t\}_i,\{\epsilon_i^{t-1}\}_i,\cdots,\{x_i^0\}_i,\{\epsilon_i^0\}_i]=0,\] \[\mathbb{E}[\exp(\tau\epsilon_i^t)|\{x_i^t\}_i,\{\epsilon_i^{t-1}\}_i,\cdots,\{x_i^0\}_i,\{\epsilon_i^0\}_i] \le \exp(\phi^2\tau^2)\,.\] 
\end{assumption}

\begin{remark}
Taking $\mu(z) = z$ gives a linear model and $\mu(z) = (1+e^{-z})^{-1}$ gives a logistic model. Also note that the assumptions imply there exists $\eta > 0$ such that $\mu'(x^{\top} \beta) \ge \eta$ for all $x, \beta \in\mathbb{R}^d$ satisfying $\|x\|_2 \le B$ and $\|\beta\|_2 \le M$ (see Lemma~\ref{lemma:eta} in Appendix~\ref{sec:appendix_proof} for a short proof).
% The bound on $|\epsilon_i|\leq C$ implies that $\epsilon$ is conditionally sub-gaussian.
%Also note that the assumption on $\mu$ implies that it is Lipschitz with constant $L$ (e.g. in the case of logistic model we have Lipschitz constant $L=1/4$). 
\end{remark}

We are interested in learning a strategy that can identify all the feature vectors $x \in \mathbb{R}^d$ that have response $y$ above some pre-specified cutoff $c$, while making as few mistakes as possible along the sequential learning process compared to the Bayes-optimal oracle that knows $\beta^*$ (i.e., the classifier $x \mapsto \mathbbm{1}\{\mu(x^{\top}\beta^*) \ge c\}$). It is worth noting that we don't make any distributional assumption on the feature vectors $x \in \mathbb{R}^d$. Thus, our adaptive algorithm works in both the adversarial setting and the stochastic setting where the features are drawn i.i.d. from some unknown underlying distribution. 

Our goal is to minimize the objective formally defined in Definition~\ref{def:regret}, which penalizes exactly when the model performs an incorrect prediction compared to the Bayes-optimal decision rule, and the penalty is the distance of the expected response value for that example to the desired cutoff $c$. %That way, it can be seen as a comparison against the Bayes-optimal oracle which has the knowledge of $\beta^*$, and the penalty is proportional to how incorrect the learner is. \qijia{Rewrite the last sentence if possible - come back to this.}
% \[r_t := \sum_{i=1}^N |\mu(x_i^{t\top}\beta^*)-c|\cdot \mathbbm{1}\Big\{\mathbbm{1}\{\mu(x_i^{t\top}\beta^*)>c\}\neq \mathbbm{1}\{\mu(x_i^{t\top}\beta_t)>c\}\Big\}\] 
\begin{definition}[One-Sided Loss]\label{def:regret} For feature-action pairs $(x_i^t,a_i^t)_{i=1}^N \in \mathbb{R}^d\times \{0,1\}$, the one-sided loss incurred at time $t$ on a batch of size $N$ with cutoff at $c$ is the following: 
\begin{equation}
\label{eqn:instant_regret}
r_t := \sum_{i=1}^N |\mu(x_i^{t\top}\beta^*)-c|\cdot \mathbbm{1}\Big\{\mathbbm{1}\{\mu(x_i^{t\top}\beta^*)>c\}\neq a_i^t\Big\}.
\end{equation}
\end{definition}
 %Typical choices of $\ell: \mathbb{R} \rightarrow \mathbb{R}$ are for example:
% \begin{align*}
%     \ell(x) = | x |\quad\quad
%     \ell(x) = (-x)_{+} 
% \end{align*}
%\begin{remark}
%We emphasize that this doesn't quite fit into the usual setting considered in the bandit literature due to the imbalance of information for the two actions in the following sense: if we choose to observe ($a_i^t=1$), we have the full information for both of the two actions (i.e., can evaluate the counter-factuals), whereas if we don't then no information whatsoever is gathered about $\beta^*$, and we don't get to observe the corresponding instantaneous regret. It is for this reason that one can only perform empirical minimization on $(x_i^t,y_i^t)$ for which $a_i^t=1$.
%\end{remark}

We give an illustrative example of how this objective naturally arises in practice. Suppose that a company is looking to hire job applicants, where each applicant will contribute some variable amount of revenue to the company and the cost of hiring an applicant is a fixed cost of $c$. If the company makes the {\it correct} decision on each applicant, it will incur no loss, where correct means that it hired exactly the applicants whose expected revenue contribution to the company is at least $c$. The company incurs loss whenever it makes an incorrect decision: if it hires an applicant whose expected revenue is below $c$, it is penalized on the difference. Likewise, if it doesn't hire an applicant whose expected revenue is above $c$, it is also penalized for the expected profit that could have been made. Moreover, this definition of loss promotes a notion of {\it individual fairness} because it encourages the decision maker to not hire an unqualified applicant over a qualified one. While our setup captures scenarios beyond fairness applications, this aspect of individual fairness in one-sided learning may be of independent interest.

\newlength\myindent
\setlength\myindent{2em}
\newcommand\bindent{%
  \begingroup
  \setlength{\itemindent}{\myindent}
  \addtolength{\algorithmicindent}{\myindent}
}
\newcommand\eindent{\endgroup}

\section{PASSIVE LEARNER HAS SLOW RATE}
\label{sec:slower_convergence}

In this section, we show that under the stronger i.i.d data generation assumption, in order to achieve asymptotically vanishing loss, one could leverage an ``offline" algorithm that learns on an initial training set only, but at the cost of having a slower rate for the one-sided loss we are interested in. Our passive learner (Algorithm~\ref{alg:offline}) proceeds by predicting positively on the first $K+S$ samples to collect the labeled examples to fit on, where the first $K$ samples are used to obtain a finite set of models which represent all possible binary decision combinations on these $K$ samples that could have been made by the GLM model. The entire observed $K+S$ labeled examples are then used to pick the best model from this finite set to be used for the remaining rounds without further updating. 

More formally, we work with the setting where the feature-utility pairs $(x_t,u_t)\sim \mathcal{P}$ are generated i.i.d in each round. Let the class of strategies be $\Pi = \{\pi^{\beta}: \|\beta\|_2 \leq M\}$, where $\pi^\beta(x):=\mathbbm{1}\{\mu(x^{\top}\beta) \geq c\}$ is the threshold rule corresponding to parameter $\beta$. Moreover, let the utility for covariate $x_t$ with action $a_t\in \{0,1\}$ be 
\[u_t(x_t,a_t) :=|y_t-c|\cdot \mathbbm{1}\Big\{\mathbbm{1}\{y_t>c\}\neq a_t\Big\}\, .\]
% $u_t(x_t,\pi^{\beta^t}) = -(\mu(x_t^{\top}\beta^t)-y_t)^2$ if $\pi^{\beta^t}(x_t) = 1$ and $u_t(x_t,\pi^{\beta^t}) = 0$ if $\pi^{\beta^t}(x_t) = 0$, respectively. 
The initial discretization of the strategy class is used for a covering argument, the size of which is bounded with VC dimension. Using Hoeffding's inequality and a union bound over $|\hat{\Pi}|$, one can easily obtain a high-probability deviation on the quantity \[\left|\frac{1}{K+S}\sum_{t=1}^{K+S} u_t(x_t,\hat{\pi}(x_t)) - \mathbb{E}_{\mathcal{P}(u,x)}[ u(x,\hat{\pi}(x))] \right|\]
uniformly over all $\hat{\pi}\in \hat{\Pi}$, after which the optimality of $\hat{\pi}_K$ is invoked for reaching the final conclusion. We show that with optimal choices of $K$ and $S$, Algorithm~\ref{alg:offline} has suboptimal guarantees -- needing as many as $\widetilde{\mathcal{O}}(1/\epsilon^3)$ rounds in order to attain an average one-sided loss of at most $\epsilon$, whereas our adaptive algorithm to be introduced later will only need $\widetilde{\mathcal{O}}(1/\epsilon^2)$ rounds. This suggests the importance of having the algorithm {\it actively} engaging throughout the data streaming process, beyond working with large collection of observational data only, for efficient learning.
\begin{algorithm}[t]
\caption{\textsc{Passive Learner}}
\label{alg:offline}
\begin{algorithmic}
   \State \textbf{Inputs:} Discretization sample size $K$, Exploration sample size $S$, cutoff $c$, Time horizon $T$ 
   %\State \textbf{Inputs:} Time horizon $T$ %, confidence $\delta\in(0,1)$
   \State \textbf{Initialization:} Choose to observe pairs of $(x_i,y_i) \in \mathbb{R}^{d}\times \mathbb{R}$ for $K+S$ rounds, set the action $a_i=1$. %if observed and $a_i=0$ otherwise.
  \begin{enumerate}
   \item Construct discretized strategy class $\hat{\Pi}$ using the first $K$ samples, containing one representative $\hat{\beta}^k \in \mathbb{R}^d$ for each element of the set $\{(\pi(x_1),\cdots, \pi(x_K)): \pi\in\Pi\}$.
   \item Find the best strategy on the observed $K+S$ data pairs as:
   \[\hat{\pi}_K^{\hat{\beta}^*} = \arg\min_{\pi\in\hat{\Pi}}\, \sum
   _{t=1}^{K+S} u_t(x_t,\pi(x_t))\]
\end{enumerate}
  
\For{$t = K+S+1,\cdots, T$}
\State Output $a_t = \hat{\pi}_K^{\hat{\beta}^*}(x_t)= \mathbbm{1}\{\mu(x_t^{\top}\hat{\beta}^*) \geq c\}$ as decision on $x_t$, observe $y_t$ if $a_t = 1$
\EndFor
\State \textbf{Output: $\hat{\beta}^*, \{a_t\}_t$}
\end{algorithmic}
\end{algorithm}
We give the guarantee in the proposition below. The proof is in Appendix~\ref{sec:appendix_offline}. 
\begin{prop}[Bound for Algorithm~\ref{alg:offline}]
\label{prop:offline}
Under Assumption~1-4 and the additional assumption that the feature-utility pairs $(x_t,u_t)\sim \mathcal{P}$ are drawn i.i.d in each round, we have that picking $K=\mathcal{O}(T^{1/3})$, $S=\mathcal{O}(T^{2/3})$ in Algorithm~\ref{alg:offline}, for $\Ct= LBM+\gamma+c+\phi\sqrt{\log(2T/\delta)}$, with probability at least $1-2\delta$, 
%and let $V(\pi):=\mathbb{E}_{x,u} [ u(x,\pi(x))]$ denote the expected utility by choosing $\pi$ that
\begin{align*}
\sum_{t=1}^T \mathbb{E}_{\mathcal{P}}[u(x,a_t)] \leq \min_{\pi \in \Pi} &\sum_{t=1}^T \mathbb{E}_{\mathcal{P}}[u(x,\pi(x))]\\
&+\mathcal{O}\Big(\Ct T^{2/3}d\log\Big(\frac{T}{d\delta}\Big)\Big). 
\end{align*}
 This in turn gives the following one-sided loss bound with the same probability: 
\[ \mathbb{E}\Big[\sum_{t=1}^Tr_t\Big] \leq \mathcal{O}\Big(\Ct T^{2/3}d\log\Big(\frac{T}{d\delta}\Big)\Big).\]
%\[\sum_{t=1}^T V(\hat{\pi}_K^{\hat{\beta}^*}) \leq T\cdot \min_{\pi \in \Pi}V(\pi)+ T^{2/3} \, .\]
% \heinrich{are we dropping the $t$ subscript in $u_t$? and $x_t$?}
\end{prop}

\section{GREEDY ACTIVE LEARNER MAY NOT CONVERGE}
\label{sec:greedy_lower}

In this section, we show that the greedy active learner, which updates the model after each round on the received labeled examples without regards for one-sided feedback, can fail to find the optimal decision rule, even under the i.i.d data assumption. More specifically, the greedy learner fits parameter $\hat{\beta}$ that minimizes the empirical loss $\sum_{(x_t, y_t)\colon a_t=1} \ell( x_t, y_t; \beta)$ on the datapoints whose labels it has observed so far at each time step. For example in the case $\mu(z) = z$ we use $\ell(x, y; \beta) = (x^\top \beta - y)^2 $, the squared loss; when $\mu(z) = \left( 1+ e^{-z}\right)^{-1}$ we instead use $\ell(x,y; \beta) = -y\log( \mu(  x^\top \beta)) - (1-y)\log(1-\mu(x^\top \beta))$, the cross-entropy loss. An alternative definition of the greedy learner can utilize the decision rule mandated by the $\hat{\beta}$ that minimizes the one-sided loss (Definition~\ref{def:regret}) on the datapoints predicted positive thus far. In our setup this is possible because whenever a datapoint label is revealed, the loss incurred by the decision can be estimated. As it turns out, these two methods share similar behavior, and we refer to reader to Appendix~\ref{sec:appendix_lower} for the discussion of this alternative method.

We illustrate in Theorem~\ref{theorem:lower_bound} below that even when allowing warm starting with full-rank randomly drawn i.i.d samples, there are settings where the greedy learner will fail to converge. More specifically, if the underlying data distribution produces with constant probability a vector $v$ with the rest of the mass concentrated on the orthogonal subspace, under Gaussian noise assumption, the prediction $\mu( v^{\top}\hat{\beta})$ has Gaussian distribution centered at the true prediction $\mu(v^{\top}\beta^{*})$. Using the Gaussian anti-concentration inequality from Lemma \ref{lemma::gaussian_lower_bound} provided in Appendix~\ref{sec:appendix_lower}, we can show that if $\mu(v^{\top}\beta^{*})$ is too close to the decision boundary $c$, there is a constant probability that the model will predict $\mu( v^{\top}\hat{\beta})<c$, and therefore the model may never gather more information in direction $v$ for updating its prediction as no more observation will be made on $v$'s label from this point on. This situation can arise for instance when dealing with a population consisting of two subgroups having small overlap between their features. 

% when the distribution has two or more distinct eigenmodes for their covariances,  The following Theorem shows that in full generality, even allowing for warm starting, the greedy learner can incur in linear regret. 
% We write the argument for the linear regression model, although our proofs can be easily extended to more general GLMs. 

%%%%%%% pushed to appendix %%%%%%%%%
%  \begin{figure}[H]
%   \begin{center}
%     \includegraphics[width=0.3\textwidth]{figures/linear_counter.png}
%     \includegraphics[width=0.32\textwidth]{figures/logistic_counter.png}
%     \end{center}
%   \caption{{\bf Examples when greedy fails to converge.} We use the example provided in Theorem~\ref{theorem:lower_bound} with $n=10000$ and $d=20$. The $x$ axis shows the number of rounds ($t$) (i.e., number of batches) and the $y$ axis shows the average regret $R_t/t$. Batch size is chosen to be $1$ for linear regression and $100$ for logistic regression. The average regret fails to decrease for the greedy method but our method (Algorithm~\ref{alg:UCB}) exhibits vanishing regret.}
% 	\label{fig:lower_bound}
% \end{figure}

 \begin{theorem}[Non-Convergence for Greedy Learner]
 \label{theorem:lower_bound}
Let $y = \mu(x^{\top} \beta^*) + \epsilon$ with $\epsilon \sim \mathcal{N}(0,1)$ and independent of $x$. Moreover, for $v \in \mathbb{R}^d$, let $P$ be a distribution such that $P(v) = 1/10$ and for all other vectors $v' \sim P$, it holds that $v'^{\top} v  = 0$. Consider an MLE fit using $\ell(x,y;\beta)$ with $n$ pairs of i.i.d. samples from $P$ for warm starting the greedy learner. Under the additional assumption that $x_1, \cdots, x_n$ span all of $\mathbb{R}^d$, if $ \mu(v^{\top} \beta^*)  = c + \tau$, with $\tau \leq 1/\sqrt{n'}$ (where $n'$ is the number of samples among $\{(x_i, y_i)\}_{i=1}^n$ with $x_i = v$), the loss after round $T$ is lower bounded as: 
\begin{equation*}
    \mathbb{E}\left[     \sum_{t=1}^T r_t \right] \geq  \Omega( (T-n) \tau)\, .
\end{equation*}
\end{theorem}

\newcommand{\pluseq}{\mathrel{+}=}

\section{ADAPTIVE ALGORITHM}
\label{sec:adaptive}
 We propose Algorithm~\ref{alg:UCB} with the goal of minimizing the cumulative one-sided loss at time horizon $T$, $R_T := \sum_{t=1}^T r_t$, independent of the data distribution at each round. %We recast the problem as a generalized linear contextual bandit problem where we choose one of the $2^N$ choices in each round, corresponding to the decision on each one of the $N$ data points in the batch. Each context is equipped with a feature matrix of size $N\times d$ where each row is either $[1,0_d]$ or $[0,x_i^t]$. In this case, for $c' = \mu^{-1}(c)$, the optimal policy chooses between $\mu(1\cdot c'+0_d^{\top}\beta^*)$ and $\mu(0\cdot c'+x_i^{t\top}\beta^*)$ to decide whether to observe $y_i^t$ or not for each $i\in[N]$ at round $t$. This turns our definition of regret in \eqref{eqn:instant_regret} as linear reward over vector $[c',\beta^*]$, from which we build upon~\citet{glm_bandit} for the analysis of the regret bound.
 \begin{algorithm}[t]
\caption{\textsc{Adaptive One-sided Batch Alg.}}
\label{alg:UCB}
\begin{algorithmic}
   \State \textbf{Inputs:} Batch size $N$, initialization sample size $K\geq d+1$ and eigenvalue $\lambda_0>0$, cutoff $c$ 
   \State \textbf{Inputs:} Lipschitz constant $L$, norm bounds $M, B, \phi, \eta, \gamma$, time horizon $T$, confidence level $\delta\in(0,1\wedge d/e)$
   \State \textbf{Initialization:} Choose to observe $K$ pairs of $\{(x_i^0,y_i^0)\}_{i=1}^K \in \mathbb{R}^{d}\times \mathbb{R}$, set $A \leftarrow \sum_{i=1}^K x_i^{0} x_i^{0\top}$
   \State Set $\kappa=\sqrt{3+2\log(1+2N B^2/\lambda_0)}$
\For{$t = 1,\cdots, T$}
\State Solve for $\hat{\beta}_t\in\mathbb{R}^d$ such that 
\begin{equation}
\label{eqn:mle_fit}    
\sum_{i=0}^{t-1}  X_i^{\top}(y_i-\mu(X_i\hat{\beta}_t)) = 0_d
\end{equation}
\If{$\|\hat{\beta_t}\|_2\leq M$}  $\beta_t \leftarrow \hat{\beta_t}$
\Else \, Perform projection step on $\hat{\beta}_t$ as
\[\beta_t =\argmin_{\|\beta\|_2 \leq M}\, \Big\|\sum_{i=1}^{t-1}X_i^{\top}\mu(X_i\beta)-\sum_{i=1}^{t-1}X_i^{\top}\mu(X_i\hat{\beta}_t)\Big\|_{A^{-1}}\]
\EndIf
\State Set $\rho_t(\delta) = \frac{2L}{\eta}\kappa \Ct \sqrt{2d\log t}\sqrt{\log(2dT/\delta)}$
\State Initialize $X_t,y_t= \emptyset$
\For{$j = 1,\cdots, N$}
\If{$\mu(x_j^{t\top} \beta_t)-c+\rho_t(\delta)\sqrt{x_j^{t\top} A^{-1}x_j^t} > 0$}
\State Choose to observe $y_j^t$ and 
set $a_j^t=1$ 
%\If{$\mu(x_j^{t\top} \beta^*) > c$ and we chose not to observe} Accumulate regret $r_t \pluseq \mu(x_j^{t\top} \beta^*) - c$
%\ElsIf{$\mu(x_j^{t\top} \beta^*) < c$ and we chose to observe} Accumulate regret $r_t \pluseq c-\mu(x_j^{t\top} \beta^*)$ 
%\Else\, Accumulate regret $r_t = 0$
%\EndIf
\State Update $X_t \leftarrow [X_t; x_j^t]\, , y_t = [y_t; y_j^t]$ 
\State Let $A\leftarrow A+x_j^tx_j^{t\top}$
\EndIf
\EndFor
\EndFor
\State \textbf{Output: $\beta_T, \{a_j^t\}$}
\end{algorithmic}
\end{algorithm}
The algorithm proceeds by first training a model on an initial labeled sample with the assumption that after initialization, the empirical covariance matrix $A$ is invertible with the smallest eigenvalue $\lambda_0>0$. At each time step, we solve for the MLE fit $\hat{\beta}_t$ on the examples observed so far, using e.g. Newton's method. If $\|\hat{\beta_t}\|_2$ is too large, we perform a projection step -- this step is only required as a analysis artifact to ensure that $\mu'(\cdot) > 0$ whenever it is evaluated in the algorithm. The model then produces point estimate $\mu(x^{\top} \beta_t)$ for each example $x$ in the current batch.
 
From here, we adopt an adaptive approach based on the point-wise uncertainty in the prediction, which for data point $x$ is proportional to $\sqrt{x^{\top} A^{-1} x}$ (where $A$ is the covariance matrix of the labeled examples the model is fit on thus far).
%Since $X_t$ is the solution to the following optimization problem at step $t$,
% \[\argmax_{\substack{X\in\mathbb{R}^{N\times (d+1)}\\ x_i\in\{[0;x_i^t],[1;0_d]\}\,\forall i\in[N]}} \, 1^{\top}\mu(X\tilde{\beta}_{t})+\rho_{t}(\delta/2T)\cdot \sum_{i=1}^N\sqrt{x_i^{[2:d+1]\top}A_{t-1}^{-1}x_i^{[2:d+1]}}\]
This choice is justified by showing that for any $X\in \mathbb{R}^{N\times (d+1)}$, whose rows consist of either $[0;x_i^t]$ or $[1;0_d]\, ,\forall i\in[N]$, we have with high probability
\[|1^{\top}\mu(X\tilde{\beta}^*) - 1^{\top}\mu(X\tilde{\beta}_t)|\leq \rho_t(\delta) \cdot\sum_{i=1}^N\sqrt{\bar{x}_i^{\top}A_{t-1}^{-1}\bar{x}_i}\]
for $\tilde{\beta}^* := [\mu^{-1}(c);\beta^*]$ the parameter of the optimal predictor and $\tilde{\beta}_{t} := [\mu^{-1}(c);\beta_{t}]$ our current best guess, where $\bar{x}_i$ is the last $d$ coordinates of the $i$-th row of the matrix $X$. With this on hand, a short calculation reveals that the loss incurred at all time step $t\leq T$, with probability at least $1-\delta$, is upper bounded as
\[r_t \leq 2\rho_t(\delta/2T) \cdot\sum_i\sqrt{x_{t,i}^{\top}A_{t-1}^{-1}x_{t,i}}\]
for $x_{t,i}$ the $i$-th row of $X_t$.
% \begin{align*}
% r_t &= 1^{\top}(\mu(X^*_t\tilde{\beta}^*) - \mu(X_t\tilde{\beta}^*)) \\
% &= 1^{\top}(\mu(X_t^*\tilde{\beta}^*)-\mu(X_t^*\tilde{\beta}_t))+1^{\top}(\mu(X_t^*\tilde{\beta}_t) - \mu(X_t\tilde{\beta}_t)) + 1^{\top}(\mu(X_t\tilde{\beta}_t) - \mu(X_t\tilde{\beta}^*))\\
% &\leq \zeta_t^{\bar{X}^*_t}(\delta/2T)+\zeta_t^{\bar{X}_t}(\delta/2T)+1^{\top}(\mu(X_t^*\tilde{\beta}_t) - \mu(X_t\tilde{\beta}_t)) \\
% &= \zeta_t^{\bar{X}^*_t}(\delta/2T)+\zeta_t^{\bar{X}_t}(\delta/2T)+1^{\top}\mu(X_t^*\tilde{\beta}_t)+\zeta_t^{\bar{X}^*_t}(\delta/2T) -1^{\top} \mu(X_t\tilde{\beta}_t)-\zeta_t^{\bar{X}_t^*}(\delta/2T)\\
% &\leq \zeta_t^{\bar{X}^*_t}(\delta/2T)+\zeta_t^{\bar{X}_t}(\delta/2T)+1^{\top}\mu(X_t\tilde{\beta}_t)+\zeta_t^{\bar{X}_t}(\delta/2T) -1^{\top}\mu(X_t\tilde{\beta}_t)-\zeta_t^{\bar{X}_t^*}(\delta/2T)\\
% &=2\zeta_t^{\bar{X}_t}(\delta/2T)
% \end{align*}
It only remains to upper bound $\sum_i\|x_{t,i}\|_{A_{t-1}^{-1}}$, for which matrix determinant lemma is invoked for volume computation of matrices under low-rank updates.
%, multiplied by a slowly increasing factor in order to balance the exploration and exploitation for the best regret guarantee. 

Intuitively, the algorithm chooses to observe the samples for which either we can't yet make a confident decision, by which collecting the sample would greatly reduce the uncertainty in that corresponding subspace (as manifested by reduction in $\sum_i\|x_{t,i}\|_{A_{t-1}^{-1}}$ for future rounds after updating the model at the end of the batch); or we are confident that the response of the sample is above $c$ (for which current decision would incur small loss).  
% The algorithm then predicts positively if the upper bound estimate is above the threshold $c$ and receive the labels of all of the positively predicted examples in the batch at once for updating the model in the next round. 
% instead of treating $\beta_t$ as fixed point estimator, Algorithm~\ref{alg:UCB} treats it as a distribution with mean centered at $\beta_t$, i.e., $\mathcal{N}(\beta_t, A^{-1})$. Therefore instead of having decision rule like $\mathbbm{1}\{x^{\top}\beta_t \geq c\}$ which is what greedy does; it accounts for the uncertainty in estimation by employing \[\mathbbm{1}\{\mathbb{P}(x^{\top}\beta_t \geq c)\geq \eta\} \Leftrightarrow \mathbbm{1}\{\mathbb{P}(\mathcal{N}(x^{\top}\beta_t, x^{\top}A^{-1}x)\geq c)\geq \eta\} \Leftrightarrow \mathbbm{1}\{x^{\top}\beta_t + \Phi^{-1}(\eta)\sqrt{x^{\top}A^{-1}x} \geq c\}\]
%  for $\Phi^{-1}$ the inverse CDF for Gaussian. 
%
% \begin{color}{red}
% \[r_t := \sum_{i=1}^N \ell(\mu(x_i^{t\top}\beta^*)-c)\cdot \mathbbm{1}\Big\{\mathbbm{1}\{\mu(x_i^{t\top}\beta^*)>c\}\neq \mathbbm{1}\{\mu(x_i^{t\top}\beta_t)>c\}\Big\}\] 
% Typical choices of $\ell: \mathbb{R} \rightarrow \mathbb{R}$ are for example:
% \begin{align*}
%     \ell(x) = | x |\\
%     \ell(x) = (-x)_{+} 
% \end{align*}
% \end{color}
%
We give the following result whose proof is in Appendix~\ref{sec:appendix_proof}. %\qijia{Comment on computational complexity.}
\begin{theorem}[Guarantee for Algorithm~\ref{alg:UCB}]
\label{thm:ucb} Suppose that Assumption 1-4 hold.
Given a batch size $N$, we have that for all $T\geq 1$, 
\[R_T\leq\widetilde{\mathcal{O}} \Big(\Ct K+\frac{L}{\eta}\Ct \sqrt{Ts}dN\Big)  \]
with probability at least $1-2\delta$ for $ 0 < \delta < \min\{1,d/e\}$, where $s = \min(N,d)$, $\Ct = LBM+\gamma+c+\phi\sqrt{\log(2T/\delta)}$ and $\widetilde{\mathcal{O}}$ hides poly-logarithmic factors in $T,\delta^{-1},d,N,B, \lambda_0$.
% \[C\cdot K+2\rho_T(\delta/2T) \sqrt{TN}\cdot\sqrt{2dNs\log\Big(\frac{\lambda_0+B^2NT}{d}\Big)-2dNs\log(\lambda_0)}\]
\end{theorem}

  \begin{figure}[t]%[H]
   \begin{center}
     \includegraphics[width=0.29\textwidth]{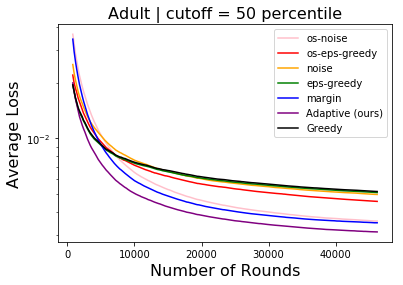}
     \includegraphics[width=0.29\textwidth]{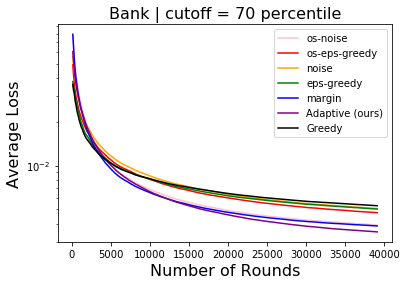}
     \end{center}
   \caption{Average one-sided loss $R_t/t$ for OLS. Each round consists of presenting a batch of $1$ example. All methods are under optimal tuning averaged across $10$ runs. %Performance of various methods (under respective tuned settings) measured by average regret , where the $x$-axis is the number of rounds $t$, where each round consisted of presenting a batch of $N = 1$ example, and the $y$-axis is the average one-sided loss $R_t/t$.  Due to space, we only show Adult with $c = 0.5$ and Bank Marketing with $c=0.7$.
   The rest of the charts are in Appendix~\ref{sec:appendix_simulation}.}
 	\label{fig:linear_plots}
 \end{figure}
 
   \begin{figure}[t]
   \begin{center}
     \includegraphics[width=0.3\textwidth]{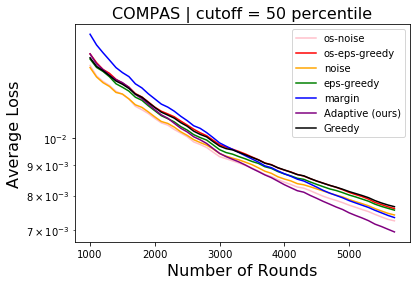}
     \includegraphics[width=0.3\textwidth]{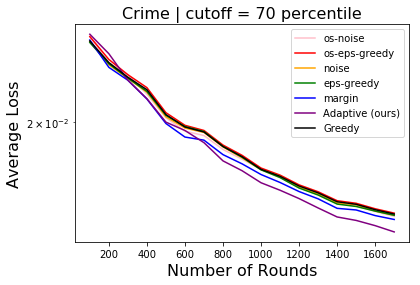}
     \end{center}
   \caption{Average one-sided loss $R_t/t$ for Logistic. Each round consists of presenting batch of $100$ samples. All methods are under optimal tuning averaged across $10$ runs. The rest of the charts are in Appendix~\ref{sec:appendix_simulation}. }
   %Performance of various methods (under respective tuned settings) measured by average regret ,
% We show the performance of the various methods (under the respective tuned settings). The $x$-axis is the number of rounds $t$, where each round consisted of presenting a batch of $N = 1000$ examples for the two datasets shown here, and the $y$-axis is the average regret $R_t/t$.  Due to space, we only show Adult with $c = 0.5$ and magic04 with $c=0.7$. The rest of the charts can be found in Appendix~\ref{sec:appendix_simulation}.}
 	\label{fig:logistic_plots}
 \end{figure}

\begin{table*}[t]
\centering
\begin{tabular}{c|c|c|c|c|c|c|c|c}
\hline
  Dataset & cutoff & greedy & $\epsilon$-grdy & os-$\epsilon$-grdy   & noise           & os-noise        & margin & ours        \\ \hline\hline
   \multirow{2}{*}{Adult}& 50\% & 239.45 & 236.34 & 211.74 & 230.77 & 165.77 & 162.31 & {\bf 144.92} \\ \cline{2-9}
&70\% & 134.74 & 134.18 & 133.8 & 131.66 & 132.39 & 132.67 & {\bf 129.81}         \\ \hline
\multirow{2}{*}{Bank}& 50\% & 164.23 & 162.67 & 117.86 & 136.0 & 88.49 & 86.26 & {\bf 74.64 }        \\ 
\cline{2-9}
& 70\% & 
 207.6 & 197.0 & 185.9 & 198.66 & 153.3 & 150.75 & {\bf 137.24 } \\ \hline
\multirow{2}{*}{COMPAS}& 50\% & 41.56 & 36.67 & 36.93 & 36.93 & 28.09 & 28.12 & {\bf 26.01}       \\ 
\cline{2-9}
& 70\% & 41.66 & 39.16 & 39.61 & 39.87 & 38.03 & 36.98 & {\bf 34.07} \\ \hline
\multirow{2}{*}{Crime}& 50\%  & 15.77 & 15.77 & 15.5 & 15.66 & 14.93 & 14.73 & {\bf 13.95}   \\ 
\cline{2-9}
&  70\% & 22.0 & 21.75 & 21.99 & 20.33 & 20.63 & 20.1 & {\bf 19.19}  \\ \hline
\multirow{2}{*}{German}& 50\% & 14.7 & 14.51 & 14.12 & 13.62 & 11.12 & 10.52 & {\bf 9.63} \\ 
\cline{2-9}
& 70\% & 15.89 & 15.53 & 15.93 & 15.41 & 14.09 & 14.52 & {\bf 13.07}        \\ \hline
\multirow{2}{*}{Blood}& 50\% & 2.06 & 2.06 & 2.06 & 2.06 & 1.92 & 1.72 & {\bf 1.52}  \\ 
\cline{2-9}
& 70\% & 3.7 & 2.78 & 3.04 & {\bf 2.38} & 3.13 & 3.06 & 2.65       \\ \hline
\multirow{2}{*}{Diabetes}& 50\% & 4.17 & 4.16 & 4.23 & 3.94 & 3.81 & 3.95 & {\bf 3.61} \\ 
\cline{2-9}
& 70\% & 6.05 & 5.56 & 6.14 & 6.05 & 5.6 & 5.39 & {\bf 5.33}      \\ \hline
\multirow{2}{*}{EEG Eye }& 50\% &  256.47 & 200.04 & 175.8 & 173.52 & 106.26 &{\bf 96.85} & 119.7 \\ 
\cline{2-9}
& 70\% & 175.71 & 167.94 & 168.73 & 157.68 & 167.52 & 160.76 & {\bf 155.79}       \\ \hline
\multirow{2}{*}{Australian} & 50\% &  3.74 & 3.74 & 3.77 & 3.63 & 3.0 & 2.79 & {\bf 2.65} \\ 
\cline{2-9}
& 70\% & 6.77 & 6.77 & 6.77 & 6.66 & 5.09 & 5.26 & {\bf 4.65}       \\ \hline
\multirow{2}{*}{Churn} & 50\% &   46.98 & 43.65 & 30.65 & 36.64 & 21.24 & 18.83 & {\bf 14.89} \\ 
\cline{2-9}
& 70\% & 49.99 & 47.84 & 47.91 & 49.89 & 41.18 & 36.17 & {\bf 35.27}     \\ \hline
\end{tabular}
\vspace{0.2cm}
\caption{{Experimental results for cumulative one-sided loss for Linear Regression}.\label{tab:linear_hyper}}%  Averaged across $10$ splits of the dataset, and for each method (except greedy), we tuned $\alpha$ over a grid of powers of $2$. We see that in 7 out of 10 cases, our adaptive method performs the best. The full chart can be found in the Appendix.
\end{table*}

%\vspace{-0.2cm}
\begin{table*}[t]
\centering
\begin{tabular}{c|c|c|c|c|c|c|c|c}
\hline
   & cutoff & greedy & $\epsilon$-grdy & os-$\epsilon$-grdy   & noise           & os-noise    & margin    & ours        \\ \hline\hline
   \multirow{2}{*}{Adult} & 50\% & 43.48 & 43.55 & 43.48 & 43.35 & 43.41 & 43.38 & {\bf 42.63} \\ \cline{2-9}
   & 70\% & 102.86 & 102.86 & 102.9 & 102.6 & 102.81 & 102.47 & {\bf 100.06} \\ \hline
    \multirow{2}{*}{Bank} & 50\% & 23.22 & 23.26 & {\bf 23.18} & 23.3 & 23.33 & 23.2 & 23.23 \\ \cline{2-9}
   & 70\% & 85.72 & 85.94 & 85.67 & 85.51 & {\bf 85.26} & 85.27 & 85.75 \\ \hline
 \multirow{2}{*}{COMPAS}& 50\% &  44.47 & 43.88 & 44.15 & 43.07 & 42.11 & 42.64 & {\bf 40.34}      \\ 
\cline{2-9}
& 70\% & 43.7 & 43.59 & {\bf 43.41} & 43.66 & 43.83 & 43.7 & 43.7  \\ \hline
\multirow{2}{*}{Crime}& 50\%  & 11.04 & 10.83 & 11.04 & 10.85 & 10.33 & 10.44 & {\bf 9.42}    \\ 
\cline{2-9}
&  70\% & 26.05 & 25.93 & 26.13 & 25.94 & 25.84 & 25.55 & {\bf 24.46} \\ \hline
\multirow{2}{*}{German}& 50\% & 35.71 & 35.21 & 33.55 & 33.35 & 24.19 & 23.19 & {\bf 20.33}  \\ 
\cline{2-9}
& 70\% & 42.55 & 41.14 & 42.18 & 40.98 & 40.64 & 40.3 & {\bf 37.12}        \\ \hline
\multirow{2}{*}{Blood}& 50\% & 5.05 & 5.05 & 4.87 & 4.83 & 4.71 & 4.53 & {\bf 4.24} \\ 
\cline{2-9}
& 70\% & 13.04 & 13.04 & 13.03 & 13.04 & 10.84 & 12.14 & {\bf 9.69}    \\ \hline
\multirow{2}{*}{Diabetes}& 50\% & 28.23 & 28.23 & 27.75 & 27.22 & 26.67 & 26.18 & {\bf 25.16}  \\ 
\cline{2-9}
& 70\% &  29.36 & 28.0 & 27.79 & 28.0 & {\bf 27.4} & 27.9 & 28.11     \\ \hline
\multirow{2}{*}{EEG Eye }& 50\% &  239.33 & 238.92 & 239.09 & 236.65 & 200.61 & 201.51 & {\bf 187.28}  \\ 
\cline{2-9}
& 70\% & 209.48 & 207.89 & 208.83 & 206.63 & 204.94 & 205.4 & {\bf 199.04}     \\ \hline
\multirow{2}{*}{Australian} & 50\% &  21.88 & 21.88 & 21.87 & 21.21 & 21.76 & 20.81 & {\bf 20.38} \\ 
\cline{2-9}
& 70\% & 17.47 & 17.29 & 17.46 & {\bf 16.49} & 17.24 & 17.46 & 17.43    \\ \hline
\multirow{2}{*}{Churn} & 50\% & 61.04 & 57.74 & 54.13 & 53.85 & 39.46 & 38.88 & {\bf 34.89}\\ 
\cline{2-9}
& 70\% &  122.96 & 117.49 & 116.04 & 112.36 & 94.61 & 88.3 & {\bf 82.23}  \\ \hline
\end{tabular}
\vspace{0.2cm}
\caption{{Experimental results for cumulative one-sided loss for Logistic Regression}.\label{tab:logistic_hyper}  }
%Averaged across $10$ splits of the dataset, and for each method (except greedy), we tuned $\alpha$ over a grid of powers of $2$. We see that in 8 out of 10 cases, our adaptive method performs the best. The full chart can be found in the Appendix.}
\end{table*}
% The model's score we used is $\hat{y}_1 - \hat{y}_0 + 0.5$, where $\hat{y}_0$ and $\hat{y}_1$ are the model's softmax probability scores for negative and positive, respectively.   We used batch size of $N=1000$ for Adult, Bank Marketing, and magic04 and $N=100$ for Crime and COMPAS.  Due to space, we only show the results for $c \in \{0.5, 0.7\}$ here for each dataset and the rest can be found in the Appendix. The percentage of datapoints whose score under the logistic model is above the $c$ cutoff is also in Appendix~\ref{sec:appendix_simulation}.
%\vspace{-0.4cm}

\section{SGD UNDER ONE-SIDED FEEDBACK}
\label{sec:sgd}
In this section, we explore learning the parameter $\beta^*$ with iterative updates under one-sided feedback. 
We consider running projected SGD with the following gradient update on $(x_t,y_t)$ at time step $t$ for the GLM model:
\begin{align*}
\beta_{t+1} = \mathcal{P}_{\Omega} &\big(\beta_t-\eta
\cdot (-y_tx_t+\mu(x_t^{\top}\beta_t)x_t) \\
&\cdot \mathbbm{1}\{\mu(x_t^{\top}\beta_t)+s_t \geq c\} \big)\, 
\end{align*}
for some exploration bonus $s_t$ to be specified later, where the projection assures that $\mu'(\cdot)\geq\gamma$ for some $\gamma > 0$ throughout the execution of the algorithm. For example in logistic regression, we project onto the convex set $\Omega := \{\beta\colon |x_t^{\top}\beta| \leq r\}$ at each step $t$ to maintain this. In words, we perform prediction on $x_t$ with the current parameter $\beta_t$, and take a stochastic projected gradient step on the sample if $\mu(x_t^{\top}\beta_t)+s_t \geq c$. Since the noise $\epsilon_t$ is assumed to be zero-mean and independent of $\beta_t$ and $x_t$, this implies that in expectation (condition on $\beta_t$), we have
\begin{align*}
\beta_{t+1}-\beta^* &= \mathcal{P}_{\Omega}\big(\beta_t-\beta^*-\eta\cdot (-\mu(x_t^{\top}\beta^*)x_t\\
&\quad +\mu(x_t^{\top}\beta_t)x_t)\cdot \mathbbm{1}\{\mu(x_t^{\top}\beta_t)+s_t \geq c\}\big)\\
&= \mathcal{P}_{\Omega}\big(\beta_t-\beta^*-\eta\cdot\mu'(z)x_t^{\top}(\beta_t-\beta^*)x_t\\
&\quad \cdot \mathbbm{1}\{\mu(x_t^{\top}\beta_t)+s_t \geq c\}\big) 
\end{align*}
where we used mean value theorem for some $z\in [x_t^{\top}\beta^*, x_t^{\top}\beta_t]$. %
% So if $\mu(x_t^{\top}\beta_t)+s_t\geq c$, we can rewrite the update as (for $\hat{\beta}$ the minimizer of $h(\cdot)$)
% \begin{align*}
% \beta_{t+1}-\beta^* &= \beta_t-\beta^*-\eta(\nabla h(\beta_t)-\nabla h(\beta^*)) - \eta(\nabla f_t(\beta_t)-\nabla h(\beta_t))-\eta(\nabla h(\beta^*)-\nabla h(\hat{\beta}))\\
% &= (I-\eta \nabla^2 h (z_t)) (\beta_t-\beta^*)- \eta(\nabla f_t(\beta_t)-\nabla h(\beta_t))-\eta\nabla h(\beta^*)
% \end{align*}
Taking norms on both sides and using the fact that convex projection is a contractive mapping, we have at step $t$, the expected progress as:
\begin{equation}
\begin{aligned}
\label{eqn:iterate_contract}
\|\beta_{t+1}-\beta^*\|_2^2 &\leq  \|\beta_t-\beta^*-\eta\cdot\mu'(z)x_t^{\top}(\beta_t-\beta^*)x_t\|_2^2 \\
% &= \|\beta_t-\beta^*\|_2^2-2\eta \cdot\mu'(z)(x_t^{\top}(\beta_t-\beta^*))^2\\
% &\quad +\eta^2\cdot \mu'(z)^2 (x_t^{\top}(\beta_t-\beta^*))^2\|x_t\|_2^2 \\
&= \|\beta_t-\beta^*\|_2^2+\big[\eta^2\cdot \mu'(z)^2\|x_t\|_2^2\\
&\quad -2\eta\cdot\mu'(z)\big](x_t^{\top}(\beta_t-\beta^*))^2
\end{aligned}
\end{equation}
if $\mu(x_t^{\top}\beta_t)+s_t\geq c$; and contraction ratio of 1 (i.e., no update on $\beta$) if $\mu(x_t^{\top}\beta_t)+s_t < c$. This suggests that in the case where we choose to accept, either $|x_t^{\top}(\beta_t-\beta^*)|$ is small, in which case the probability of making a mistake on this sample is small already; or if large we make sufficient progress in this direction by performing the update. This is formalized in Algorithm~\ref{alg:sgd} and the corresponding Proposition~\ref{prop:sgd} below, whose proof we defer to Appendix~\ref{sec:iterative_method}. In order to have any hope of making progress towards $\beta^*$ (i.e., observing $y_t$ with non-trivial probability), however, we make the following assumption on the feature vectors.
\begin{assumption}[Subgaussian i.i.d Features]
\label{assume:feature}
The feature vectors $x_t$ at each time step $t$ are drawn i.i.d with independent $\sigma$-sub-gaussian coordinates. This in turn implies that since $\mu(x^{\top}\beta^*)$ is a univariate L-lipschitz function of $\|\beta^*\|_2\sigma$-subgaussian random variable, $\mu(x_t^{\top}\beta^*)- \mathbb{E}_x[\mu(x^{\top}\beta^*)]$ is itself $CL\|\beta^*\|\sigma$-subgaussian for some numerical constant $C$.
\end{assumption}

\begin{algorithm}[t]
\caption{\textsc{SGD Under Partial Feedback}}
\label{alg:sgd}
\begin{algorithmic}
\State \textbf{Inputs:} Initial $\beta_0$ and $d_0$ such that $\|\beta_0-\beta^*\| \leq d_0$
\State \textbf{Inputs:} Accuracy $\alpha$, Lipschitz const $L$, param $\delta$
\State \textbf{Inputs:} Bound $B$ such that $|\epsilon_t| \leq B\; \forall t$
\For{$t = 0,\cdots, T$}
\State Set $s_t = L\cdot (1+\delta)\cdot d_t\|x_t\|_2$
\If{$\mu(x_t^{\top}\beta_t)+s_t < c$}
\State Don't accept $x_t$, keep $\beta_{t+1}=\beta_t$ and $d_{t+1} = d_t$
\Else\, Accept $x_t$ and receive label $y_t$
\If{$\left|y_t-\mu(x_t^{\top}\beta_t)\right| \leq \alpha+B$}
%\State Don't perform gradient update 
\State Set $\beta_{t+1}=\beta_t$ and $d_{t+1}=d_t$
\Else\, Update as
$\beta_{t+1} = \mathcal{P}_{\Omega}(\beta_t-(L\|x_t\|_2^2)^{-1}\cdot (-y_tx_t+\mu(x_t^{\top}\beta_t)x_t))$; set $d_{t+1}^2 = d_t^2 - \alpha^2\|x_t\|_2^{-2} L^{-2}$
\EndIf
\EndIf
\EndFor
\State \textbf{Output: $\beta_T$}
\end{algorithmic}
\end{algorithm}

\begin{prop}
\label{prop:sgd}
Under Assumption 1 and 5, given $\rho\in(0,1)$, we have with probability at least $1-\rho$, for cutoff $c = \mathbb{E}_x[\mu(x^{\top}\theta^*)]-\zeta$ with $\zeta \geq \sqrt{2L\|\beta^*\|^2\sigma^2\log(\rho^{-1})}$, at iteration $t$ of Algorithm~\ref{alg:sgd}, either 
\[\mathbb{E}[\|\beta_{t+1}-\beta^*\|^2] \leq \mathbb{E}[\|\beta_t-\beta^*\|_2^2] - \frac{\alpha^2}{\|x_t\|_2^2 L^2}\, ,\]
or 
\[
\left|\mu(x_t^{\top}\beta^*)-\mu(x_t^{\top}\beta_t)\right| \leq L\gamma^{-1} (\alpha+2B)
\]
if picking $\delta^{-1} =\rho-e^{-\frac{\zeta^2}{2L^2\|\beta^*\|^2\sigma^2}}$.
Moreover, the probability of making a misclassification error at time step $t$ satisfies
\[\mathbb{P}\left(\mathbbm{1}\{\mu(x_t^{\top}\beta^*) \geq c\} \neq \mathbbm{1}\{\mu(x_t^{\top}\beta_t) +s_t\geq c\} \right) \leq \rho\, .\]
\end{prop}

\begin{remark}
If we are interested in cutoff $c = \mathbb{E}[\mu(x^{\top}\beta^*)]+\zeta$ for some $\zeta > 0$, a similar argument shows that picking $s_t = -L\cdot(1+\delta)\cdot d_t\|x_t\|_2$ will give the same misclassification error probability $\rho$, with the exception of course being that we won't be able to get the high probability contraction ratio for $\|\beta_t-\beta^*\|_2$ due to the lack of observations on $y_t$.
\end{remark}

\section{EXPERIMENTS}
\label{sec:experiments}
To further support our theoretical findings and demonstrate the effectiveness of our algorithm in practice, we test our method on the following datasets:\\  
1. {\bf Adult} \cite{lichman2013uci} ($48842$
examples). The task is to predict whether the person's income is more than $50$k.\\
2. {\bf Bank Marketing} \cite{lichman2013uci} ($45211$ examples). Predict if someone will subscribe to a bank product.\\
3. {\bf ProPublica’s COMPAS} \cite{propublica2018}  ($7918$ examples). Recidivism data.\\
%The task is to predict recidivism based on 
%criminal history, jail and prison time, demographics,
%and risk scores.\\
4. {\bf Communities and Crime} \cite{lichman2013uci} ($1994$ examples). Predict if community is high (>70\%tile) crime.\\
%(above the $70$-th percentile)
5. {\bf German Credit} \cite{lichman2013uci} ($1000$ examples). Classify into good or bad credit risks.\\
6. {\bf Blood Transfusion Service Center} \cite{OpenML2013} ($784$ examples). Predict if person donated blood. \\
7. {\bf Diabetes} \cite{OpenML2013} ($768$ examples). Detect if patient shows signs of diabetes. \\ 
8. {\bf EEG Eye State} \cite{OpenML2013} ($14980$ examples). Detect if eyes are open or closed based on EEG data. \\
9. {\bf Australian Credit Approval} \cite{OpenML2013} ($690$ examples). Predict for credit card approvals.\\
10. {\bf Churn} \cite{OpenML2013} ($5000$ examples). Determine whether or not the customer churned. 
%5. {\bf magic04} \cite{lichman2013uci} ($19020$ examples). Telescope data where the goal is to predict whether there is gamma signal or background noise.

We compare against the following baselines:\\
1. {\bf  Greedy}, where we perform least-squares/logistic fit $\beta_t$ on the collected data and predict positive/observe label if $\mu(x^{\top}\beta_t)>c$.\\
2. {\bf $\epsilon$-Greedy} \cite{sutton2018reinforcement}, which with probability $\alpha/\sqrt{t}$, we make a random decision on the prediction (with equal probability), otherwise we use the greedy approach.\\
3. {\bf One-sided $\epsilon$-Greedy}, which with probability $\alpha/\sqrt{t}$ we predict positively, otherwise we use the greedy approach. This baseline is inspired from ideas in the original apple tasting paper \cite{helmbold2000apple}.\\
4. {\bf Noise}, which we add $\alpha u / \sqrt{t}$ to the prediction where $u$ is drawn uniformly on $[-\frac{1}{2}, \frac{1}{2}]$.\\
5. {\bf One-sided Noise}, which we add $\alpha u / \sqrt{t}$ to the prediction where $u$ is drawn uniformly on $[0, 1]$.\\
6. {\bf Margin}, which we add $\alpha / \sqrt{t}$ to the prediction. This can be seen as a non-adaptive version of our approach, since the quantity we add to the prediction for this baseline is uniform across all points.

For each dataset, we take all the examples and make a random stratified split so that $5\%$ of the data is used to train the initial model and the rest is used for online learning. For the linear regression experiments, we used a batch size of $1$ while for logistic regression we used a batch size of $1000$ for Adult, Bank, EEG Eye State and $100$ for the rest due to computational costs of retraining after each batch using \texttt{scikit-learn}'s implementation of logistic regression.  We compute the loss based on using an estimated $\beta^*$ obtained by fitting the respective model (either linear or logistic) on the entire dataset. Due to space limitation, we only show the results for cutoff $c$ chosen so that $50\%$ and $70\%$ of the data points are below the cutoff w.r.t. $\beta^*$ in Table~\ref{tab:linear_hyper} for linear regression and Table~\ref{tab:logistic_hyper} for logistic regression. Full results are in Appendix~\ref{sec:appendix_simulation}. For each dataset and setting of $c$, we averaged the performance of each method across $10$ different random splits of the dataset and tuned $\alpha$ over a grid of powers of $2$ (except greedy). %The final result is reported below. %where $\alpha$ trades-off exploration vs. exploitation. \qijia{Either we keep this paragraph here and adjust the captions in the figure/table below (which I think is a good idea), or we leave this out. Otherwise it's a bit repetitive.}

%In Table~\ref{tab:linear_hyper} and Table~\ref{tab:logistic_hyper}, we show for each dataset and $c\in\{0.5,0.7\}$ the cumulative regret incurred by each method averaged across $10$ runs under optimal tuning, for linear and logistic regression respectively. We see that our method outperforms the other baselines on 7 of 10 settings for the former and 10 out of 10 for the latter. 
%In Figure~\ref{fig:linear_plots} and Figure~\ref{fig:logistic_plots}, we plot the decay of the average regret over time for a number of datasets and $c$ settings, demonstrating that adaptive method have a faster learning rate and make fewer mistakes along the learning process.

%
\vspace{-0.1cm}
\section{DISCUSSION}
Many machine learning systems learn under active one-sided feedback, where experimental design is intertwined with the decision making process. In such scenarios, the data collection is informed by past decisions and can be inherently biased. In this work, we show that without accounting for such biased sampling, the model could enter a feedback loop that only reinforce its past misjudgements, resulting in a strategy that may not align with the long term learning goal. Indeed, we demonstrate that the de facto default approach (i.e., greedy or passive learning) often yields suboptimal performance when viewed through this lens. Thus, we propose a natural objective for the one-sided learner and give a practical algorithm that can be used to avoid such undesirable downstream effects. Both the theoretical grounding and the empirical effectiveness of the proposed algorithm offer evidence that it serves as a much better alternative in such settings. Future work involves building upon these ideas to incorporate more general model classes such as neural networks.

%\qijia{Include these?}
%Such a method may have wide implications to machine learning as the one-sided learning setting appears in many applications and we provide a new tool practitioners can leverage in these situations.

% We show that without any intervention, such bias may cause certain individuals or sub-populations to be disadvantaged and our method may be used to alleviate such effects.

\newpage
\bibliographystyle{plainnat}
\bibliography{paper}

\begin{thebibliography}{33}
\providecommand{\natexlab}[1]{#1}
\providecommand{\url}[1]{\texttt{#1}}
\expandafter\ifx\csname urlstyle\endcsname\relax
  \providecommand{\doi}[1]{doi: #1}\else
  \providecommand{\doi}{doi: \begingroup \urlstyle{rm}\Url}\fi

\bibitem[Antos et~al.(2013)Antos, Bart{\'o}k, P{\'a}l, and
  Szepesv{\'a}ri]{antos2013toward}
Andr{\'a}s Antos, G{\'a}bor Bart{\'o}k, D{\'a}vid P{\'a}l, and Csaba
  Szepesv{\'a}ri.
\newblock Toward a classification of finite partial-monitoring games.
\newblock \emph{Theoretical Computer Science}, 473:\penalty0 77--99, 2013.

\bibitem[Bart{\'o}k and Szepesv{\'a}ri(2012)]{bartok2012partial}
G{\'a}bor Bart{\'o}k and Csaba Szepesv{\'a}ri.
\newblock Partial monitoring with side information.
\newblock In \emph{International Conference on Algorithmic Learning Theory},
  pages 305--319. Springer, 2012.

\bibitem[Bechavod et~al.(2019)Bechavod, Ligett, Roth, Waggoner, and
  Wu]{bechavod2019equal}
Yahav Bechavod, Katrina Ligett, Aaron Roth, Bo~Waggoner, and Steven~Z Wu.
\newblock Equal opportunity in online classification with partial feedback.
\newblock In \emph{Advances in Neural Information Processing Systems}, pages
  8972--8982, 2019.

\bibitem[Berk(2017)]{berk2017impact}
Richard Berk.
\newblock An impact assessment of machine learning risk forecasts on parole
  board decisions and recidivism.
\newblock \emph{Journal of Experimental Criminology}, 13\penalty0 (2):\penalty0
  193--216, 2017.

\bibitem[Beygelzimer et~al.(2011)Beygelzimer, Langford, Li, Reyzin, and
  Schapire]{sauer}
Alina Beygelzimer, John Langford, Lihong Li, Lev Reyzin, and Robert Schapire.
\newblock Contextual bandit algorithms with supervised learning guarantees.
\newblock In Geoffrey Gordon, David Dunson, and Miroslav Dudík, editors,
  \emph{Proceedings of the Fourteenth International Conference on Artificial
  Intelligence and Statistics}, volume~15 of \emph{Proceedings of Machine
  Learning Research}, pages 19--26, Fort Lauderdale, FL, USA, 2011. PMLR.

\bibitem[Bordes et~al.(2005)Bordes, Ertekin, Weston, and
  Bottou]{bordes2005fast}
Antoine Bordes, Seyda Ertekin, Jason Weston, and L{\'e}on Bottou.
\newblock Fast kernel classifiers with online and active learning.
\newblock \emph{Journal of Machine Learning Research}, 6\penalty0
  (Sep):\penalty0 1579--1619, 2005.

\bibitem[Cesa-Bianchi et~al.(2006{\natexlab{a}})Cesa-Bianchi, Gentile, and
  Zaniboni]{cesa2006worst}
Nicolo Cesa-Bianchi, Claudio Gentile, and Luca Zaniboni.
\newblock Worst-case analysis of selective sampling for linear classification.
\newblock \emph{Journal of Machine Learning Research}, 7\penalty0
  (Jul):\penalty0 1205--1230, 2006{\natexlab{a}}.

\bibitem[Cesa-Bianchi et~al.(2006{\natexlab{b}})Cesa-Bianchi, Lugosi, and
  Stoltz]{cesa2006regret}
Nicolo Cesa-Bianchi, G{\'a}bor Lugosi, and Gilles Stoltz.
\newblock Regret minimization under partial monitoring.
\newblock \emph{Mathematics of Operations Research}, 31\penalty0 (3):\penalty0
  562--580, 2006{\natexlab{b}}.

\bibitem[Chan et~al.(1999)Chan, Fan, Prodromidis, and
  Stolfo]{chan1999distributed}
Philip~K Chan, Wei Fan, Andreas~L Prodromidis, and Salvatore~J Stolfo.
\newblock Distributed data mining in credit card fraud detection.
\newblock \emph{IEEE Intelligent Systems and Their Applications}, 14\penalty0
  (6):\penalty0 67--74, 1999.

\bibitem[Chu et~al.(2011)Chu, Zinkevich, Li, Thomas, and
  Tseng]{chu2011unbiased}
Wei Chu, Martin Zinkevich, Lihong Li, Achint Thomas, and Belle Tseng.
\newblock Unbiased online active learning in data streams.
\newblock In \emph{Proceedings of the 17th ACM SIGKDD international conference
  on Knowledge discovery and data mining}, pages 195--203, 2011.

\bibitem[Covington et~al.(2016)Covington, Adams, and Sargin]{covington2016deep}
Paul Covington, Jay Adams, and Emre Sargin.
\newblock Deep neural networks for youtube recommendations.
\newblock In \emph{Proceedings of the 10th ACM conference on recommender
  systems}, pages 191--198, 2016.

\bibitem[Ensign et~al.(2017)Ensign, Friedler, Neville, Scheidegger, and
  Venkatasubramanian]{ensign2017decision}
Danielle Ensign, Sorelle~A Friedler, Scott Neville, Carlos Scheidegger, and
  Suresh Venkatasubramanian.
\newblock Decision making with limited feedback: Error bounds for recidivism
  prediction and predictive policing.
\newblock 2017.

\bibitem[Filippi et~al.(2010)Filippi, Cappe, Garivier, and
  Szepesv\'{a}ri]{glm_bandit}
Sarah Filippi, Olivier Cappe, Aur\'{e}lien Garivier, and Csaba Szepesv\'{a}ri.
\newblock Parametric bandits: The generalized linear case.
\newblock In J.~D. Lafferty, C.~K.~I. Williams, J.~Shawe-Taylor, R.~S. Zemel,
  and A.~Culotta, editors, \emph{Advances in Neural Information Processing
  Systems 23}, pages 586--594. 2010.

\bibitem[Hardt et~al.(2016)Hardt, Price, and Srebro]{hardt2016equality}
Moritz Hardt, Eric Price, and Nati Srebro.
\newblock Equality of opportunity in supervised learning.
\newblock In \emph{Advances in neural information processing systems}, pages
  3315--3323, 2016.

\bibitem[He et~al.(2014)He, Pan, Jin, Xu, Liu, Xu, Shi, Atallah, Herbrich,
  Bowers, et~al.]{he2014practical}
Xinran He, Junfeng Pan, Ou~Jin, Tianbing Xu, Bo~Liu, Tao Xu, Yanxin Shi,
  Antoine Atallah, Ralf Herbrich, Stuart Bowers, et~al.
\newblock Practical lessons from predicting clicks on ads at facebook.
\newblock In \emph{Proceedings of the Eighth International Workshop on Data
  Mining for Online Advertising}, pages 1--9, 2014.

\bibitem[Helmbold et~al.(2000)Helmbold, Littlestone, and
  Long]{helmbold2000apple}
David~P Helmbold, Nicholas Littlestone, and Philip~M Long.
\newblock Apple tasting.
\newblock \emph{Information and Computation}, 161\penalty0 (2):\penalty0
  85--139, 2000.

\bibitem[Jindal and Liu(2007)]{jindal2007review}
Nitin Jindal and Bing Liu.
\newblock Review spam detection.
\newblock In \emph{Proceedings of the 16th international conference on World
  Wide Web}, pages 1189--1190, 2007.

\bibitem[Kou et~al.(2014)Kou, Peng, and Lu]{kou2014mcdm}
Gang Kou, Yi~Peng, and Chen Lu.
\newblock Mcdm approach to evaluating bank loan default models.
\newblock \emph{Technological and Economic Development of Economy}, 20\penalty0
  (2):\penalty0 292--311, 2014.

\bibitem[Lee et~al.(2014)Lee, Yoon, Shim, Vasseur, and
  Demonceaux]{lee2014local}
Unghui Lee, Sangyol Yoon, HyunChul Shim, Pascal Vasseur, and Cedric Demonceaux.
\newblock Local path planning in a complex environment for self-driving car.
\newblock In \emph{The 4th Annual IEEE International Conference on Cyber
  Technology in Automation, Control and Intelligent}, pages 445--450. IEEE,
  2014.

\bibitem[Lichman et~al.(2013)]{lichman2013uci}
Moshe Lichman et~al.
\newblock Uci machine learning repository, 2013.

\bibitem[Lu et~al.(2016)Lu, Zhao, and Hoi]{lu2016online}
Jing Lu, Peilin Zhao, and Steven~CH Hoi.
\newblock Online passive-aggressive active learning.
\newblock \emph{Machine Learning}, 103\penalty0 (2):\penalty0 141--183, 2016.

\bibitem[Paden et~al.(2016)Paden, {\v{C}}{\'a}p, Yong, Yershov, and
  Frazzoli]{paden2016survey}
Brian Paden, Michal {\v{C}}{\'a}p, Sze~Zheng Yong, Dmitry Yershov, and Emilio
  Frazzoli.
\newblock A survey of motion planning and control techniques for self-driving
  urban vehicles.
\newblock \emph{IEEE Transactions on intelligent vehicles}, 1\penalty0
  (1):\penalty0 33--55, 2016.

\bibitem[Pazzani and Billsus(2007)]{pazzani2007content}
Michael~J Pazzani and Daniel Billsus.
\newblock Content-based recommendation systems.
\newblock In \emph{The adaptive web}, pages 325--341. Springer, 2007.

\bibitem[Perdomo et~al.(2020)Perdomo, Zrnic, Mendler-D{\"u}nner, and
  Hardt]{perdomo2020performative}
Juan~C Perdomo, Tijana Zrnic, Celestine Mendler-D{\"u}nner, and Moritz Hardt.
\newblock Performative prediction.
\newblock \emph{arXiv preprint arXiv:2002.06673}, 2020.

\bibitem[ProPublica(2018)]{propublica2018}
ProPublica.
\newblock Compas recidivism risk score data and analysis, Mar 2018.
\newblock URL
  \url{https://www.propublica.org/datastore/dataset/compas-recidivism-risk-score-data-and-analysis}.

\bibitem[Sculley(2007)]{sculley2007practical}
D~Sculley.
\newblock Practical learning from one-sided feedback.
\newblock In \emph{Proceedings of the 13th ACM SIGKDD international conference
  on Knowledge discovery and data mining}, pages 609--618, 2007.

\bibitem[Srivastava et~al.(2008)Srivastava, Kundu, Sural, and
  Majumdar]{srivastava2008credit}
Abhinav Srivastava, Amlan Kundu, Shamik Sural, and Arun Majumdar.
\newblock Credit card fraud detection using hidden markov model.
\newblock \emph{IEEE Transactions on dependable and secure computing},
  5\penalty0 (1):\penalty0 37--48, 2008.

\bibitem[Sutton and Barto(2018)]{sutton2018reinforcement}
Richard~S Sutton and Andrew~G Barto.
\newblock \emph{Reinforcement learning: An introduction}.
\newblock MIT press, 2018.

\bibitem[Tiwari(2018)]{tiwari2018machine}
Abhishek~Kumar Tiwari.
\newblock Machine learning application in loan default prediction.
\newblock \emph{Machine Learning}, 4\penalty0 (5), 2018.

\bibitem[Tollenaar and Van~der Heijden(2013)]{tollenaar2013method}
Nikolaj Tollenaar and PGM Van~der Heijden.
\newblock Which method predicts recidivism best?: a comparison of statistical,
  machine learning and data mining predictive models.
\newblock \emph{Journal of the Royal Statistical Society: Series A (Statistics
  in Society)}, 176\penalty0 (2):\penalty0 565--584, 2013.

\bibitem[Tsai and Chen(2010)]{tsai2010credit}
Chih-Fong Tsai and Ming-Lun Chen.
\newblock Credit rating by hybrid machine learning techniques.
\newblock \emph{Applied soft computing}, 10\penalty0 (2):\penalty0 374--380,
  2010.

\bibitem[Vanschoren et~al.(2013)Vanschoren, van Rijn, Bischl, and
  Torgo]{OpenML2013}
Joaquin Vanschoren, Jan~N. van Rijn, Bernd Bischl, and Luis Torgo.
\newblock Openml: Networked science in machine learning.
\newblock \emph{SIGKDD Explorations}, 15\penalty0 (2):\penalty0 49--60, 2013.
\newblock \doi{10.1145/2641190.2641198}.
\newblock URL \url{http://doi.acm.org/10.1145/2641190.2641198}.

\bibitem[Wang et~al.(2010)Wang, Mathieu, Ke, and Cai]{wang2010predicting}
Ping Wang, Rick Mathieu, Jie Ke, and HJ~Cai.
\newblock Predicting criminal recidivism with support vector machine.
\newblock In \emph{2010 International Conference on Management and Service
  Science}, pages 1--9. IEEE, 2010.

\end{thebibliography}

\clearpage
\newpage
\onecolumn
\appendix
%%%%%%%%%%% appendix A starts here %%%%%%%%%%%%%%%
\section{Proof for Algorithm~\ref{alg:offline}}
\label{sec:appendix_offline}
We start by stating a lemma below involving our assumption on $\mu(\cdot),\beta,x$ that implies the instantaneous one-sided loss is bounded by $\Ct:= LBM+\gamma+c+\phi\sqrt{\log(2T/\delta)}$ with probability $1-\delta/T$.
\begin{lemma} Given $\|x\|_2\leq B$, $\|\beta\|_2\leq M$, $\epsilon$ zero-mean subgaussian with parameter $\phi$ and $\mu(\cdot)$ has Lipschitz constant $L$ with $\mu(0)\leq \gamma$, we have the following tail bound on the empirical instantaneous one-sided loss
\[u(x,a) :=|\mu(x^{\top}\beta^*)+\epsilon-c|\cdot \mathbbm{1}\Big\{\mathbbm{1}\{\mu(x^{\top}\beta^*)+\epsilon>c\}\neq a\Big\}\, ,\]
as
\[\mathbb{P}(u(x,a) \geq LBM+\gamma+c+\tau) \leq 2\exp\Big(-\frac{\tau^2}{\phi^2}\Big)\]
for all $\tau>0$. 
\end{lemma}
\begin{proof}
By Cauchy-Schwartz and Lipschitz assumption, we have $|\mu(x^{\top}\beta^*)-c| \leq LBM+\gamma+c$. Now by triangle inequality since $|u|\leq LBM+\gamma+c+|\epsilon|$ and $\epsilon$ is subgaussian with parameter $\phi$, 
\[\mathbb{P}(|u|\geq LBM+\gamma+c+\tau)\leq  \mathbb{P}(|\epsilon|\geq \tau) \leq 2\exp\Big(-\frac{\tau^2}{\phi^2}\Big)\]
for all $\tau>0$. %\qijia{Moreover, $u$ is itself subgaussian with parameter ...}
\end{proof}
All the analysis that follow will condition on this event, where we have that the one-sided loss is bounded by $\Ct$ for all $t\leq T$ with probability exceeding $1-\delta$. %\qijia{Added in the previous lemma. Proprogate $C_{T,\delta}$ everywhere.} 

Before proceeding, we give a lemma below that characterizes the optimal solution to the expected one-sided loss minimization problem on the population level.
\begin{lemma} 
\label{lem: optimal_policy} 
The optimal strategy for the expected one-sided loss minimization problem satisfies
\[\min_{\pi \in \Pi}  \mathbb{E}_{\mathcal{P}}[u(x,\pi(x))]=\mathbb{E}_{\mathcal{P}}\Big[|y-c|\cdot \mathbbm{1}\Big\{\mathbbm{1}\{y>c\}\neq \mathbbm{1}\{\mu(x^{\top}\beta^*)>c\}\Big\}\Big]\]
for strategy class $\Pi = \{\pi^{\beta}: \|\beta\|_2 \leq M\}$, where $\pi^\beta(x):=\mathbbm{1}\{\mu(x^{\top}\beta) > c\}$ and expectation is taken over data that follows $y=\mu(x^{\top}\beta^*)+\epsilon$. In other words, $a = \mathbbm{1}\{\mu(x^{\top}\beta^*)>c\}$ is the optimal strategy for the objective at population level. 
\end{lemma}
\begin{proof}
We can rewrite the objective in terms of $\beta$ as
\begin{align*}
    \min_{\pi \in \Pi}  \mathbb{E}_{\mathcal{P}}[u(x,&\pi(x))]
    =\min_{\beta:\|\beta\|_2\leq M} \mathbb{E}_{\mathcal{P}}\Big[(y-c)\cdot \Big(\mathbbm{1}\{y>c\}- \mathbbm{1}\{\mu(x^{\top}\beta)>c\}\Big)\Big]\\
    &= \min_{\beta:\|\beta\|_2\leq M} \mathbb{E}_{\mathcal{P}}\Big[(\mu(x^{\top}\beta^*)+\epsilon-c)\cdot \Big(\mathbbm{1}\{\mu(x^{\top}\beta^*)+\epsilon>c\}- \mathbbm{1}\{\mu(x^{\top}\beta)>c\}\Big)\Big]\, .
\end{align*}
Therefore it suffices to show that 
\[\beta^* = \argmax_{\beta:\|\beta\|_2\leq M} \mathbb{E}_{\mathcal{P}}\Big[(\mu(x^{\top}\beta^*)+\epsilon-c)\cdot  \mathbbm{1}\{\mu(x^{\top}\beta)>c\}\Big]\, .\]
As $\epsilon$ is zero-mean and independent of $x$ by assumption, we have 
\[\mathbb{E}_{\mathcal{P}}\Big[(\mu(x^{\top}\beta^*)+\epsilon-c)\cdot  \mathbbm{1}\{\mu(x^{\top}\beta)>c\}\Big] = \mathbb{E}_{\mathcal{P}}\Big[(\mu(x^{\top}\beta^*)-c)\cdot  \mathbbm{1}\{\mu(x^{\top}\beta)>c\}\Big]\, ,\]and the claim above immediately follows.
\end{proof}

With this in hand, we are ready to show the one-sided loss bound for the offline learner in Algorithm~\ref{alg:offline}.
\begin{proof}[Proof of Proposition~\ref{prop:offline}]
To construct an $\epsilon$-cover $\hat{\Pi}$ in the pseudo-metric $\rho(\pi,\hat{\pi}) = \mathbb{P}(\pi(x)\neq \hat{\pi}(x))$ for i.i.d feature-utility pairs $(x_t,u_t)\sim \mathcal{P}$, since the linear threshold functions in $\mathbb{R}^d$ has a VC dimension of $d+1$,
%$\hat{\Pi}:=\{\pi_\beta\colon \beta\in\hat{\Theta}\}$, where $\pi_\beta(x)\colon=\mathbbm{1}\{\mu(x^{\top}\beta) \geq c\}$. To construct $\epsilon$-cover $\hat{\Pi}$ in the pseudo-metric $\rho(\pi_\beta,\pi_{\beta'}) = \mathbb{P}(\pi_\beta(x)\neq \pi_{\beta'}(x))$, 
% we use the following observation that for all $\|x\|_2\leq B$, using the Lipschitz assumption $\mu(x^{\top}\beta)-\mu(x^{\top}\beta')\leq L\|x\|\|\beta-\beta'\|\leq LB\epsilon$
% \begin{align*}
% \rho(\pi_\beta,\pi_{\beta'}) &= \mathbb{P}(\mathbbm{1}\{\mu(x^{\top}\beta) \geq c\}\neq \mathbbm{1}\{\mu(x^{\top}\beta') \geq c\})\\
% &\leq \mathbb{E}[|\mathbbm{1}\{\mu(x^{\top}\beta) \geq c\}- \mathbbm{1}\{\mu(x^{\top}\beta') \geq c\}|]\\
% &\leq \|\beta-\beta'\|\leq \epsilon
% \end{align*}
a standard argument with Sauer's lemma (see e.g. \citep{sauer}) concludes that for sequences $x_1,\cdots, x_T$ drawn i.i.d, with probability $1-\delta/2$ over a random subset of size $K$,
\begin{align*}
\min_{\pi \in \hat{\Pi}} \mathbb{E}_{\mathcal{P}(u,x)}\Big[\sum_{t=1}^T u_t(x_t,\pi(x_t))\Big] \leq \min_{\pi \in \Pi} &\mathbb{E}_{\mathcal{P}(u,x)}\Big[\sum_{t=1}^T u_t(x_t,\pi(x_t))\Big] + \frac{\Ct T}{K}\Big(2(d+1) \log\Big(\frac{eT}{(d+1)}\Big)+\log\Big(\frac{2}{\delta}\Big)\Big)
\end{align*}
for $|\hat{\Pi}| = (\frac{eT}{d+1})^{(d+1)}$. Now running the passive algorithm for the discretized strategy class $\hat{\Pi}$ on the exploration data we collected in the first phase (consists of $K+S$ rounds), we have for a fixed straegy $\hat{\pi}\in\hat{\Pi}$, since the $K+S$ terms are i.i.d and unbiased, %\qijia{note we evaluate both $a_t=0$ and $a_t=1$}
\[\mathbb{E}_{\mathcal{P}(u,x)}\Big[\sum_{t=1}^{K+S}u_t(x_t,\hat{\pi}(x_t))\Big] = (K+S)\cdot\mathbb{E}_{\mathcal{P}(u,x)}[ u(x,\hat{\pi}(x))]\, .\]
Now via a Hoeffding's inequality for bounded random variables and a union bound over $|\hat{\Pi}|$, with probability at least $1-\delta$, simultaneously for all $\hat{\pi}\in\hat{\Pi}$, 
\[\Big|\frac{1}{K+S}\sum_{t=1}^{K+S} u_t(x_t,\hat{\pi}(x_t)) - \mathbb{E}_{\mathcal{P}(u,x)}[ u(x,\hat{\pi}(x))] \Big| \leq \sqrt{\frac{\Ct^2}{2(K+S)}\log\Big(\frac{2 |\hat{\Pi} |}{\delta}\Big)}\, .\]
Therefore applying the inequality twice with $\hat{\pi}_K := \hat{\pi}_K^{\hat{\beta}^*}$ and $\hat{\pi}^* := \min_{\pi \in \hat{\Pi}} \mathbb{E}_{\mathcal{P}}[ u(x,\pi(x))]$, and using the optimality of $\hat{\pi}_K$ as the empirical minimizer, we have for each round,
\[\mathbb{E}_{\mathcal{P}(u,x)}[ u(x,\hat{\pi}_K(x))] \leq \mathbb{E}_{\mathcal{P}(u,x)}[ u(x,\hat{\pi}^*(x))] + 2 \sqrt{\frac{\Ct^2}{2(K+S)}\left(\log\Big(\frac{4}{\delta}\Big)+(d+1)\log\Big(\frac{eT}{d+1}\Big)\right)}\]
%\begin{align*}
%&\geq \max_{\pi \in \hat{\Pi}}\sum_{t=1}^T u_t \cdot \pi(x_t) - \Big(CK+2 T\sqrt{\frac{2}{K}\log(\frac{4|\hat{\Pi}|}{\delta})}\Big)\\
%&= \max_{\pi \in \hat{\Pi}}\sum_{t=1}^T u_t \cdot \pi(x_t) - \Big(CK+2 T\sqrt{\frac{2}{K}\Big(\log(\frac{4}{\delta})+d\log(\frac{eT}{d})\Big)}\Big)
%\end{align*}
with probability at least $1-\delta/2$. Now summing up over $T$ rounds, putting together the inequalities established and minimizing over $K$ and $S$, gives the final utility bound as
\begin{align*}
    \sum_{t=1}^T \mathbb{E}_{\mathcal{P}}[u(x,a_t)] &\leq \min_{\pi \in \Pi} \sum_{t=1}^T \mathbb{E}_{\mathcal{P}}[u(x,\pi(x))]+T\sqrt{\frac{2\Ct^2}{(K+S)}\Big(\log\Big(\frac{4}{\delta}\Big)+(d+1)\log\Big(\frac{eT}{d+1}\Big)\Big)}\\
    &+ \Ct(K+S)+\frac{\Ct T}{K}\Big(2(d+1) \log\Big(\frac{eT}{(d+1)}\Big)+\log\Big(\frac{2}{\delta}\Big)\Big)\\ &=\min_{\pi \in \Pi} \sum_{t=1}^T \mathbb{E}_{\mathcal{P}}[u(x,\pi(x))]+\mathcal{O}\Big(\Ct T^{2/3}d\log\Big(\frac{T}{d\delta}\Big)\Big)\, ,
\end{align*} 
with probability at least $1-\delta$. This in turn gives the one-sided loss bound
\begingroup
\allowdisplaybreaks
\begin{align*}
&\sum_{t=1}^T \mathbb{E}_{\mathcal{P}}[u(x,a_t)]-\min_{\pi \in \Pi} \sum_{t=1}^T \mathbb{E}_{\mathcal{P}}[u(x,\pi(x))]\\
&=\sum_{t=1}^T \mathbb{E}_{\mathcal{P}}\Big[|y-c|\cdot \mathbbm{1}\Big\{\mathbbm{1}\{y>c\}\neq a_t\Big\}-|y-c|\cdot \mathbbm{1}\Big\{\mathbbm{1}\{y>c\}\neq \mathbbm{1}\{\mu(x^{\top}\beta^*)>c\}\Big\}\Big]\\
&= \sum_{t=1}^T \mathbb{E}_{\mathcal{P}}\Big[(y-c)\cdot \Big(\mathbbm{1}\{\mu(x^{\top}\beta^*)>c\}-a_t\Big)\Big] \\
&= \sum_{t=1}^T \mathbb{E}_{\mathcal{P}}\Big[(\mu(x^{\top}\beta^*)-c)\cdot \Big(\mathbbm{1}\{\mu(x^{\top}\beta^*)>c\}-a_t\Big)\Big] + 
\sum_{t=1}^T \mathbb{E}_{\mathcal{P}}\Big[\epsilon \cdot \Big(\mathbbm{1}\{\mu(x^{\top}\beta^*)>c\}-a_t\Big)\Big]\\
&=  \sum_{t=1}^T \mathbb{E}_{\mathcal{P}}\Big[|\mu(x^{\top}\beta^*)-c|\cdot \mathbbm{1}\Big\{\mathbbm{1}\{\mu(x^{\top}\beta^*)>c\}\neq a_t\Big\}\Big] + \sum_{t=1}^T \mathbb{E}_{\mathcal{P}}\Big[\epsilon \cdot \Big(\mathbbm{1}\{\mu(x^{\top}\beta^*)>c\}-a_t\Big)\Big]\\
&= \sum_{t=1}^T \mathbb{E}_{\mathcal{P}}[r_t] = \mathcal{O}\Big(\Ct T^{2/3}d\log\Big(\frac{T}{d\delta}\Big)\Big)
%\hspace{1cm} \text{(since $a_t$ is independent of $\epsilon_t$)}
%- \sum_{t=1}^T \mathbb{E}_{\mathcal{P}}\Big[|\epsilon|\cdot \mathbbm{1}\Big\{a_t\neq \mathbbm{1}\{\mu(x^{\top}\beta^*)>c\}\Big\}\Big]\\
\end{align*} 
\endgroup
with the same probability, where we used Lemma~\ref{lem: optimal_policy} in the first equality and the fact that $\epsilon$ is zero-mean and independent of $x$ for the last step.
%\[\sum_{t=1}^T \mathbb{E}[r_t]=\sum_{t=1}^T \mathbb{E}_{\mathcal{P}}\Big[|\mu(x^{\top}\beta^*)-c|\cdot \Big| \mathbbm{1}\{\mu(x^{\top}\beta^*)>c\}-a_t\Big|\Big]\]
\end{proof}

%%%%%%%%%%%%%%%%% Appendix B starts here %%%%%%%%%%%%%%%%
\section{Proof for Section~\ref{sec:greedy_lower}}
\label{sec:appendix_lower}
The following Gaussian anti-concentration bound is used throughout the proof.
\begin{lemma}\label{lemma::gaussian_lower_bound}
For $X\sim\mathcal{N}(\mu, \sigma^2)$, we have the lower bound on Gaussian density as:
%\begin{equation}\label{equation::lower_bound_gaussian_density}
 \[\mathbb{P}\left (X-\mu<-t \right) \ge \frac{1}{\sqrt{2\pi}}\frac{\sigma t}{t^2+\sigma^2} e^{-\frac{t^2}{2\sigma^2}}\, .\]
%\end{equation}
\end{lemma}

\subsection{Proof of Theorem~\ref{theorem:lower_bound}}

\begin{proof}[Proof of Theorem~\ref{theorem:lower_bound}]
 The optimality condition for MLE fit $\hat{\beta}$ gives that $\sum_{i=1}^{n}  y_i \cdot x_i  = \sum_{i=1}^{n}\mu(x_i^{\top}\hat{\beta})\cdot x_i$, which implies 
 \begin{align*}
    \sum_{i=1}^n x_i\cdot\epsilon_i = \sum_{i=1}^n x_i \cdot  \Big[\mu(x_i^{\top}\hat{\beta})-\mu(x_i^{\top}\beta^*)\Big]\, . %\\
    %&= \sum_{i=1}^n x_i \cdot \mu'(x_i^{\top}\bar{\beta})\cdot [x_i^{\top}\beta^*-x_i^{\top}\hat{\beta}]
     \end{align*}
In the direction $v$, where $v$ is orthogonal to every other vector drawn from $P$, we have for $n'$ the number of times $v$ has appeared in the $n$ samples used for warm-starting the greedy learner,
 \[\sum_{i=1}^{n'} v\cdot\epsilon_i = \sum_{i=1}^{n'}v\cdot \Big[\mu(v^{\top}\hat{\beta})-\mu(v^{\top}\beta^*)\Big]\]
 and therefore assuming $\epsilon_i\sim\mathcal{N}(0,1)$, we have $\mu(v^{\top}\hat{\beta})\sim \mathcal{N}(\mu(v^{\top}\beta^*),\frac{1}{n'})$. If $\mu(v^{\top}\beta^*)=c+\tau$ for $0 < \tau\le 1/\sqrt{n'}$, by anticoncentration property of the Gaussian distribution (Lemma~\ref{lemma::gaussian_lower_bound}):
 \begin{align*}
     \mathbb{P}\left( \mu(v^{\top}\hat{\beta}) < c  \right)
     &\geq \frac{1}{\sqrt{ 2\pi}} \frac{\tau/\sqrt{n'} }{1/n' + \tau^2 } \exp\left(-\frac{ \tau^2}{2/n'} \right)\geq  1/10 \, .
     %&\geq \frac{1}{\sqrt{ 2\pi}} \frac{\tau \sqrt{v^\top \left(  \sum_{i=1}^n x_ix_i^\top   \right)^{-1} v } }{v^\top \left(  \sum_{i=1}^n x_ix_i^\top   \right)^{-1} v + \tau^2 } \exp\left(-\frac{ \tau^2}{2v^\top \left(  \sum_{i=1}^n x_ix_i^\top   \right)^{-1} v } \right)\\
     %&= \frac{1}{\sqrt{2\pi}} \frac{ \gamma}{1+\gamma^2 }\exp\left(- \frac{\gamma}{2} \right) \geq  1/30
 \end{align*}
   Hence, with constant probability the greedy procedure will reject $v$ at the next round. From then on-wards, the greedy learner will reject all instances of $v$ and will incur a loss of $\tau$ every time it encounters it in any subsequent round.
%  If $\gamma = 2$ and $\omega = v$ this implies:
%  \begin{equation*}
%      \mathbb{P}\left(   \langle \hat{\beta}, v\rangle < c\right) \geq  1/30
%  \end{equation*}
%  $\mathbb{P}(\mu(v^{\top}\hat{\beta}) < c)$ with constant probability. 
%  Let $\hat{\beta}$ be the least squares estimator:
%  \begin{equation*}
%      \hat{\beta} = (\sum_{i=1}^n x_i x_i^\top )^{-1} \left( \sum_{i=1}^n x_i y_i \right)
%  \end{equation*}
%  Let's expand $\hat{\beta}$ further:
%  \begin{equation*}
%      \hat{\beta} = \beta^* + \left( \sum_{i=1}^n x_ix_i^\top \right)^{-1} \left( \sum_{i=1}^n x_i \xi_i \right)
%  \end{equation*}
%  This implies that for any fixed direction $\omega \in \mathbb{S}^d$:
%  \begin{equation*}
%      \langle \hat{\beta} , \omega \rangle = \langle \omega, \beta^* \rangle + \underbrace{\omega^\top \left( \sum_{i=1}^n x_ix_i^\top \right)^{-1} \left( \sum_{i=1}^n x_i \xi_i \right)}_{V}
%  \end{equation*}
%  A simple calculation yields $V \sim \mathcal{N}(0, \omega^\top \left(  \sum_{i=1}^n x_ix_i^\top   \right)^{-1} \omega)$.
%  If $\langle \omega, \beta^* \rangle = c + \tau$, with $\tau = \gamma \sqrt{\omega^\top \left(  \sum_{i=1}^n x_ix_i^\top   \right)^{-1} \omega}$ by anticoncentration properties of the Gaussian distribution (Lemma~\ref{lemma::gaussian_lower_bound} in Appendix~\ref{sec:appendix_lower}): 
 \end{proof}
 
%  \qijia{We can also generalize to something like this? (but don't have to)}
%  Show $\exists v \sim\mathcal{P}$ s.t. w.h.p (for $\epsilon$ gaussian and $X=U\Sigma V^{\top}$ drawn i.i.d from $\mathcal{P}$ at each time step)
%   \begin{align*}
%       \mathbb{P}(v^{\top}\hat{\beta}<c)&=\mathbb{P}(v^{\top}(X^{\top}X)^{-1}X^{\top}\epsilon < c-v^{\top}\beta^*)\\
%       &= \mathbb{P}(v^{\top}V\Sigma^{-1}U^{\top}\epsilon < c-v^{\top}\beta^*)\\
%       &= \mathbb{P}(v^{\top}V\Sigma^{-1}\epsilon' < c-v^{\top}\beta^*)
%         \end{align*}
% This suggests that if $X$ is drawn from a distribution with (1) $\lambda_{\max}$ large; (2) $\mathbb{P}(|x_i^{\top}v_{\max}| \geq 1-\alpha)$ nontrivial, then your calculation should more or less hold too? One possibility could be - fixed a basis $R \in \mathbb{R}^{d\times d}$, we require that $x_i = R g_i$ for random vector $g_i$ that is Bernoulli-Gaussian (with different variance).

\subsection{Linear One-Sided Loss for Empirical One-Sided Loss Minimization}
Let $(x)_+ = \max(0,x)$. For observed pairs of $(x_i,y_i)$ denoted by set $\mathcal{S}$, the empirical risk minimization of our one-sided loss can be rewritten as
\[\arg\min_{\beta}\; \sum_{i\in\mathcal{S}} |y_i-c|\cdot \frac{\big((x_i^{\top}\beta-c)(c-y_i)\big)_+}{(x_i^{\top}\beta-c)(c-y_i)},\]
from which it's obvious that it's not convex in $\beta$. However, in the case $x \in \{e_i\}_{i=1}^d$ (or any other orthogonal system), the coordinate decouples (after rotation) and the problem becomes solving $d$ problems in 1D as 
\[\arg\min_{\beta_i} \; \sum_{j \in \mathcal{S} : x_j = e_i} \frac{\big((\beta_i-c)(c-y_j)\big)_+}{|\beta_i-c|}\]
and the following procedure would find the optimal solution to the problem: (1) for all $y_i$ such that $y_i\leq c$, compute $l = \sum_{i:y_i\leq c} (c-y_i)$; (2) similarly compute $u = \sum_{i:y_i\geq c} (y_i-c)$; (3) compare the two quantities if $l < u$, any $\beta_i>c$ would be a global optimum for the problem with objective function value $l$, otherwise any $\beta_i<c$ would be a global optimum for the problem with objective function value $u$. Therefore for group $i$ ($x=e_i$) with response $y_i \sim \mathcal{N}(c+\tau_i,\sigma_i^2)$, where we initialize with $t$ observed samples, again using Lemma~\ref{lemma::gaussian_lower_bound},
\begin{align*}
    \mathbb{P}(u < l) &= \mathbb{P}\Big(\sum_{t:y_i^t\geq c} (y_i^t-c) <\sum_{t:y_i^t\leq c} (c-y_i^t)\Big) \\
    &= \mathbb{P}\Big(\sum_t (y_i^t-c) < 0\Big)\\
    &= \mathbb{P}\Big(\mathcal{N}(t\tau_i,t\sigma_i^2)<0\Big) \\
    &\geq \frac{1}{2\pi}\frac{\sqrt{t}\sigma_i\tau_i}{t\tau_i^2+\sigma_i^2}\exp\Big(-\frac{t\tau_i^2}{2\sigma_i^2}\Big) > 1/5
    \end{align*}
for $\tau_i=1/\sqrt{t}$ and $\sigma_i=1$, after which no observations will be made on group $i$ as $\hat{\beta}_i < c$ and linear one-sided loss will be incurred with constant probability.
%This suggests that the cost function would be piecewise-linear and inspires the following step-down procedure on $2^{|\mathcal{S}|}$.
\subsection{Simulation}
We use the example provided in Theorem~\ref{theorem:lower_bound} with $n=10000$ and $d=20$. The $x$ axis shows the number of rounds ($t$) (i.e., number of batches) and the $y$ axis shows the average one-sided loss $R_t/t$. Batch size is chosen to be $1$ for linear regression and $100$ for logistic regression. The average one-sided loss fails to decrease for the greedy method but our method (Algorithm~\ref{alg:UCB}) exhibits vanishing one-sided loss.
 \begin{figure}[H]
  \begin{center}
    \includegraphics[width=0.3\textwidth]{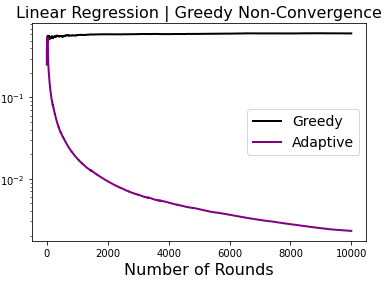}
    \includegraphics[width=0.32\textwidth]{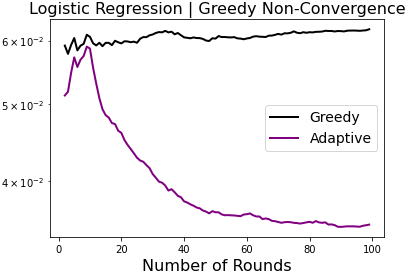}
    \end{center}
  \caption{{Examples when greedy fails to converge where th $y$-axis is the average one-sided loss up to the round indicated in the $x$-axis.} }
	\label{fig:lower_bound}
\end{figure}

%%%%%%%%%%%%%%%%%%%%% appendix C starts here %%%%%%%%%%%%%%%%%%%%%
\section{Proof for Algorithm~\ref{alg:UCB}}
\label{sec:appendix_proof}

\begin{lemma} 
\label{lemma:eta}
Suppose that Assumption~\ref{assumption:setup} holds. Let $B > 0$. Then there exists $\eta > 0$ such that $\mu'(x^{\top} \beta) \ge \eta$ for all $x, \beta \in\mathbb{R}^d$ satisfying $\|x\|_2 \le B$ and $\|\beta\|_2 \le M$.
\end{lemma}
\begin{proof}%[Proof of Lemma~\ref{lemma:eta}]
By Assumption 1, we have $\mu'$ is continuous and positive everywhere. Define the interval $\mathcal{I} := [-B\cdot M, B\cdot M]$. We have by Cauchy-Schwarz that $x^{\top}\beta \in \mathcal{I}$. Since $\mathcal{I}$ is closed and bounded, by Heine–Borel $\mathcal{I}$ is compact in $\mathbb{R}$. Since the image of a continuous function on a compact set is also compact, it follows that there exists $\eta > 0$ such that $\mu'(x) \ge \eta$ for all $x \in \mathcal{I}$, as desired.
\end{proof}

Much of the analysis in the lemma below is built upon \citep{glm_bandit}, generalized to our setting.

\begin{lemma}[Instantaneous One-Sided Loss]
\label{lem:rt}
For all $1 \leq t\leq T$ and some $0< \delta < \min(1,d/e)$, with $A_0 :=  \sum_{i=1}^K x_i^{0} x_i^{0\top}\succeq \lambda_0\cdot I_d$ and $A_t := A_0+\sum_{i=1}^t X_{i}^{\top}X_{i}$, we have under the assumption stated in Theorem~\ref{thm:ucb} that
\[r_t \leq 4\frac{L}{\eta}\kappa \Ct \sqrt{2d\log t}\sqrt{\log(2dT/\delta)}\cdot  \sum_{i=1}^N\sqrt{\bar{x}_i^{t\top}A_{t-1}^{-1}\bar{x}_i^t}\]
with probability at least $1-\delta$ for $\kappa=\sqrt{3+2\log(1+2N B^2/\lambda_0)}$, where $\bar{x}_i^t$ is either $x_i^t$ or $0_d$ depending on whether we choose to observe the context.
\end{lemma}
\begin{proof}
We recast the problem as picking 1 out of $2^N$ choices (induced by all possible binary decisions on each of the $N$ samples in the batch) in each round with linear reward function. To this end, for each feature vector $x \in \mathbb{R}^d$, we encode the algorithm's two choices as $(d+1)$-dimensional vectors $[0;x]$ (selecting it) and $[1;0_d]$ (not selecting it). Let us denote $c' = \mu^{-1}(c)$, where $c$ is the cutoff. Then for each $i\in[N]$ at round $t$, OPT chooses $x_i^t$ to predict positively and observe $y_i^t$ if $\mu([0;x_i^t]\cdot \tilde{\beta}^*) = \mu(0\cdot c'+x_i^{t\top}\beta^*)$ exceeds
$\mu([1;0_d]\cdot \tilde{\beta}^*) = \mu(1\cdot c'+0_d^{\top}\beta^*)$ where $ \tilde{\beta}^* := [c';\beta^*]$.
%$\mu(1\cdot c'+0_d^{\top}\beta^*)$ and $\mu(0\cdot c'+x_i^{t\top}\beta^*)$ to decide whether to observe $y_i^t$ or not.
We can then define the following notation $X_t$ representing Algorithm~\ref{alg:UCB}'s choices at round $t$:
% the following two optimization are equivalent
% \begin{equation}\label{eqn:opt1}
% X_t = \arg\max_{X} \max_\beta\;  1^{\top}X\beta \quad \text{s.t.}\quad  \|\beta-\beta_{t-1}\|_{A_{t-1}} \leq \rho_{t-1}(\delta)
% \end{equation}
\begin{equation}
\label{eqn:opt}
X_t := \argmax_{X\in\mathbb{R}^{N\times (d+1)}\colon x_i\in\{[0;x_i^t],[1;0_d]\}\,\forall i\in[N]} \; 1^{\top}\mu(X\tilde{\beta}_{t})+\rho_{t}(\delta/2T)\cdot \sum_{i=1}^N\sqrt{x_i^{[2:d+1]\top}A_{t-1}^{-1}x_i^{[2:d+1]}}
\end{equation}
where $\tilde{\beta}_{t} := [c';\beta_{t}]\in\mathbb{R}^{d+1}$ with $\beta_{t}$ as the MLE fit on $(x_i,y_i)$ pairs observed so far. Note that this is the same as the $X_t$ one would get from Algorithm~\ref{alg:UCB} up to padding of $0$ at front for the observed contexts and appending vector $[1,0_d]$ for the unobserved ones.

For any feasible context matrix $X\in \mathbb{R}^{N\times (d+1)}$ at round $t$, using the Lipschitz assumption on $\mu(\cdot)$, and denote $\tilde{\beta}^* = [c';\beta^*]\in\mathbb{R}^{d+1}$, we have
\begin{align*}
|1^{\top}\mu(X\tilde{\beta}^*) - 1^{\top}\mu(X\tilde{\beta}_t)|  
&\leq L\|X(\tilde{\beta}^*-\tilde{\beta}_t)\|_1 \\
&=L\|\bar{X}\beta^* - \bar{X}\beta_t\|_1
\end{align*}
where we used that both $\beta$'s have constant $c'$ in the first coordinate, so the problem is reduced to looking at the last $d$ coordinates, defined as $\bar{X}^*$ and $\bar{X}_t$ respectively. Now let $g_t(\beta):= \sum_{i=1}^{t-1}\bar{X}_i^{\top}\mu(\bar{X}_i\beta)$, Mean Value Theorem gives that
\begin{align*}
g_t(\beta^*)-g_t(\beta_t)&=\int_0^1 \nabla g_t(s\beta^*+(1-s)\beta_t) \, ds\cdot (\beta^*-\beta_t)\\
&=\int_0^1 \sum_{i=1}^{t-1}\bar{X}_i^{\top}\mu'\Big(\bar{X}_i(s\beta^*+(1-s)\beta_t)\Big)\bar{X}_i\,ds \cdot (\beta^*-\beta_t)\\
&=: G_t\cdot (\beta^*-\beta_t)
\end{align*}
where $G_t$ satisfies $G_t\succeq \eta A_{t-1} \succ 0$ since the middle term $\mu'(\cdot)$ can be seen as a diagonal matrix with entries $\geq\eta$ by Lemma~\ref{lemma:eta}. Therefore we have for some $v\in\{\pm 1\}^N$ by (1) Cauchy-Schwarz (as $G_t^{-1}\succ 0$); (2) triangle inequality; (3) $\beta_t$ is optimal for the projection problem,
\begingroup
\allowdisplaybreaks
\begin{align*}
|1^{\top}\mu(X\tilde{\beta}^*) - 1^{\top}\mu(X\tilde{\beta}_t)| &\leq L\|X(\tilde{\beta}^* - \tilde{\beta}_t)\|_1\\
&= L\|\bar{X}G_t^{-1}(g_t(\beta^*)-g_t(\beta_t))\|_1\\
&\leq L\|g_t(\beta^*)-g_t(\beta_{t})\|_{G_{t}^{-1}}\|\bar{X}^{\top}v\|_{G_t^{-1}} \\
&\leq \frac{L}{\eta}\|g_t(\beta^*)-g_t(\beta_{t})\|_{A_{t-1}^{-1}}\|\bar{X}^{\top}v\|_{A_{t-1}^{-1}} \\
&\leq \frac{L}{\eta}\|g_t(\beta^*)-g_t(\beta_{t})\|_{A_{t-1}^{-1}}\cdot  \sum_{i=1}^N\sqrt{\bar{x}_i^{\top}A_{t-1}^{-1}\bar{x}_i}\\%\sqrt{N\lambda_{\max}\{\bar{X}(A_{t-1})^{-1}\bar{X}^{\top}\}}\\
&=  \frac{L}{\eta}\Big\| \sum_{i=1}^{t-1}\bar{X}_i^{\top}\mu(\bar{X}_i\beta^*)- \sum_{i=1}^{t-1}\bar{X}_i^{\top}\mu(\bar{X}_i\beta_t)\Big\|_{A_{t-1}^{-1}}\cdot  \sum_{i=1}^N\sqrt{\bar{x}_i^{\top}A_{t-1}^{-1}\bar{x}_i}\\
&\leq \frac{L}{\eta}\Big(\Big\| \sum_{i=1}^{t-1}\bar{X}_i^{\top}\mu(\bar{X}_i\beta^*)- \sum_{i=1}^{t-1}\bar{X}_i^{\top}\mu(\bar{X}_i\hat{\beta}_t)\Big\|_{A_{t-1}^{-1}}\\
&\quad \quad +\Big\| \sum_{i=1}^{t-1}\bar{X}_i^{\top}\mu(\bar{X}_i\beta_t)- \sum_{i=1}^{t-1}\bar{X}_i^{\top}\mu(\bar{X}_i\hat{\beta}_t)\Big\|_{A_{t-1}^{-1}}\Big)\cdot  \sum_{i=1}^N\sqrt{\bar{x}_i^{\top}A_{t-1}^{-1}\bar{x}_i}\\
&= \frac{L}{\eta}\Big(\Big\| \sum_{i=1}^{t-1}\bar{X}_i^{\top}\mu(\bar{X}_i\beta^*)- \sum_{i=1}^{t-1}\bar{X}_i^{\top}y_i\Big\|_{A_{t-1}^{-1}}\\
&\quad \quad +\Big\| \sum_{i=1}^{t-1}\bar{X}_i^{\top}\mu(\bar{X}_i\beta_t)- \sum_{i=1}^{t-1}\bar{X}_i^{\top}\mu(\bar{X}_i\hat{\beta}_t)\Big\|_{A_{t-1}^{-1}}\Big)\cdot  \sum_{i=1}^N\sqrt{\bar{x}_i^{\top}A_{t-1}^{-1}\bar{x}_i}\\
&\leq \frac{2L}{\eta}\Big\| \sum_{i=1}^{t-1}\bar{X}_i^{\top}(\mu(\bar{X}_i\beta^*)-y_i)\Big\|_{A_{t-1}^{-1}}\cdot  \sum_{i=1}^N\sqrt{\bar{x}_i^{\top}A_{t-1}^{-1}\bar{x}_i}\\
&\leq \frac{2L}{\eta}\kappa \Ct \sqrt{2d\log t}\sqrt{\log(d/\delta)}\cdot  \sum_{i=1}^N\sqrt{\bar{x}_i^{\top}A_{t-1}^{-1}\bar{x}_i}\\%\sqrt{N\max_i \{x_i^{t\top}(A_{t-1})^{-1}x_i^t\}}\\
&=: \zeta_t^{\bar{X}}(\delta) =: \rho_t(\delta) \cdot\sum_{i=1}^N\sqrt{\bar{x}_i^{\top}A_{t-1}^{-1}\bar{x}_i}
\end{align*}
\endgroup
where we used Lemma 1 from~\citep{glm_bandit} for bounding the first term in the last step, which holds with probability at least $1-\delta$ for $0< \delta < \min(1,d/e)$ and $\kappa:=\sqrt{3+2\log(1+2N B^2/\lambda_0)}$. We used the fact that (1) the $N$ context vectors $\{\bar{x}_i\}$ are independent of each other in each round; (2) each of the error term $\mu(\bar{X}_i\beta^*)-y_i$ has sub-gaussian tail by assumption.

Leveraging this, for the instantaneous one-sided loss, we get with probability at least $1-\delta/T$ (denote $X^*_t\in\mathbb{R}^{N\times (d+1)}$ as the context chosen by the best action at time $t$)
\begin{align*}
r_t &= 1^{\top}(\mu(X^*_t\tilde{\beta}^*) - \mu(X_t\tilde{\beta}^*)) \\
&= 1^{\top}(\mu(X_t^*\tilde{\beta}^*)-\mu(X_t^*\tilde{\beta}_t))+1^{\top}(\mu(X_t^*\tilde{\beta}_t) - \mu(X_t\tilde{\beta}_t)) + 1^{\top}(\mu(X_t\tilde{\beta}_t) - \mu(X_t\tilde{\beta}^*))\\
&\leq \zeta_t^{\bar{X}^*_t}(\delta/2T)+\zeta_t^{\bar{X}_t}(\delta/2T)+1^{\top}(\mu(X_t^*\tilde{\beta}_t) - \mu(X_t\tilde{\beta}_t)) \\
&= \zeta_t^{\bar{X}^*_t}(\delta/2T)+\zeta_t^{\bar{X}_t}(\delta/2T)+1^{\top}\mu(X_t^*\tilde{\beta}_t)+\zeta_t^{\bar{X}^*_t}(\delta/2T) -1^{\top} \mu(X_t\tilde{\beta}_t)-\zeta_t^{\bar{X}_t^*}(\delta/2T)\\
&\leq \zeta_t^{\bar{X}^*_t}(\delta/2T)+\zeta_t^{\bar{X}_t}(\delta/2T)+1^{\top}\mu(X_t\tilde{\beta}_t)+\zeta_t^{\bar{X}_t}(\delta/2T) -1^{\top}\mu(X_t\tilde{\beta}_t)-\zeta_t^{\bar{X}_t^*}(\delta/2T)\\
&=2\zeta_t^{\bar{X}_t}(\delta/2T)
\end{align*}
where we used the optimality of $X_t$ for \eqref{eqn:opt} in the last inequality. Union bounding over $T$ time steps yields the claim.
\end{proof}

Below we state a helper lemma for bounding the second term $\sum_{i=1}^N\|\bar{x}_i^{t}\|_{A_{t-1}^{-1}}$ from the previous lemma.
\begin{lemma}[Helper Lemma]
\label{lem:det}
For all $T\geq 1$, let $A_0 =  \sum_{i=1}^K x_i^{0} x_i^{0\top}\succeq \lambda_0\cdot I_d$ and $A_t = A_0+\sum_{i=1}^t X_i^{\top}X_i$, for $s =\min(N ,d)$, under the assumption $\|x_i^t\|_2\leq B\; \forall i, t$,
\[\sum_{t=1}^T  \min\Big\{\sum_{i=1}^N\|\bar{x}_i^t\|_{A_{t-1}^{-1}}^2, N\Big\} \leq 2dNs\log\Big(\frac{\lambda_0+B^2NT}{d}\Big)-2dNs\log(\lambda_0)\, ,\]
where $\bar{x}_i^t$ is either $x_i^t$ or $0_d$ depending on whether we choose to observe the context.
\end{lemma}
\begin{proof}
We have from the Matrix Determinant Lemma that
\begin{align*}
\det(A_t)=\det(A_{t-1}+X_t^{\top}X_t)&=\det(A_{t-1})\det(I_N+X_t(A_{t-1})^{-1}X_t^{\top})\\
&=\det(A_{t-1})\prod_{i=1}^s (1+\lambda_i^t)\\
&=\det(A_0)\prod_{j=1}^t \prod_{i=1}^s (1+\lambda_i^j),
\end{align*}
where we denoted the nonzero eigenvalues of the PSD matrix $X_t(A_{t-1})^{-1}X_t^{\top}$ as $\{\lambda_i^t\}_{i=1}^s$. Since $\{\lambda_i^t\}_{i=1}^s$ are eigenvalues of $(A_{t-1})^{-1}$ restricted to span$\{x^t_i\}_{i=1}^N$, we have 
\begin{equation}
\label{eqn:eig}
\sum_{i=1}^N x^{t\top}_i (A_{t-1})^{-1} x^{t}_i \leq N\cdot\max_i(\lambda_i^t)\leq N\sum_{i=1}^s \lambda_i^t  \,.
\end{equation}

 Now using that $x\leq2\log(1+x)$ for $x\in[0,1]$, 
\[\sum_{j=1}^t \sum_{i=1}^s \min\{1/s,\lambda_i^t\} \leq 2\sum_{j=1}^t \sum_{i=1}^s \log(1+\lambda_i^t)=2(\log\det(A_t)-\log\det(A_0))\, . \]
Since $\text{Trace}(A_t) = \lambda_0+\sum_{j=1}^t \text{Trace}(X_j^{\top}X_j)\leq \lambda_0+B^2Nt$ if all covariates are bounded as $\|x_i^t\|_2\leq B$. Therefore from AM-GM inequality, since the determinant is the product of the eigenvalues, we have
\[\sum_{j=1}^t \sum_{i=1}^s \min\{1/s,\lambda_i^t\}\leq 2d\log\Big(\frac{\lambda_0+B^2Nt}{d}\Big)-2d\log(\lambda_0)\, .\]
Implying from \eqref{eqn:eig} that %\qijia{Need to be a bit more careful here moving sums inside/outside.}
\begin{align*}
\sum_{t=1}^T\min\Big\{\sum_{i=1}^N\bar{x}_i^{t\top}A_{t-1}^{-1}\bar{x}_i^t, N\Big\}
&\leq N\sum_{t=1}^T\min\Big\{\sum_{i=1}^s \lambda_i^t,1\Big\}\leq Ns\sum_{t=1}^T \sum_{i=1}^s \min\{\lambda_i^t,1/s\} \\
&\leq  2dNs\log\Big(\frac{\lambda_0+B^2NT}{d}\Big)-2dNs\log(\lambda_0)\, ,
\end{align*}
as claimed.
% Since 
% \[\sum_{i=1}^N x^{t\top}_i (A_{t-1})^{-1} x^{t}_i \leq \frac{1}{\lambda_{\min}(A_{t-1})}\sum_{i=1}^N \|x_i^t\|_2^2\leq \frac{NB^2}{\lambda_{\min}(A_{t-1})} \,,\]
% under the assumption that $\lambda_{\min}(A_{t-1}) \geq \max\{B^2, N\}$,
% \[\sum_{i=1}^N x^{t\top}_i (A_{t-1})^{-1} x^{t}_i \leq N\]
% \sum_{t=1}^T  \min\Big\{\sum_{i=1}^N\bar{x}_i^{t\top}A_{t-1}^{-1}\bar{x}_i^t, N\Big\} \leq 
\end{proof}

We are now ready to put things together to give the final one-sided loss bound for our algorithm.
\begin{proof}[Proof of Theorem~\ref{thm:ucb}]

For the cumulative one-sided loss, using Lemma~\ref{lem:rt} and Lemma~\ref{lem:det} above, and the conditional event that the instantaneous one-sided loss is bounded by $\Ct$, with probability at least $1-\delta$,
\begingroup
\allowdisplaybreaks
\begin{align*}
R_T &\leq \Ct\cdot K+\sum_{t=1}^T r_t\\
&\leq \Ct\cdot K+\sum_{t=1}^T \min\Big\{2\rho_t(\delta/2T) \cdot\sum_{i=1}^N\sqrt{\bar{x}_i^{t\top}A_{t-1}^{-1}\bar{x}_i^t}, N\Ct\Big\}\\
&\leq \Ct\cdot K+ 2\rho_T(\delta/2T)\cdot \sum_{t=1}^T  \min\Big\{\sum_{i=1}^N\sqrt{\bar{x}_i^{t\top}A_{t-1}^{-1}\bar{x}_i^t}, N\Big\}\\
&\leq \Ct\cdot K+2\rho_T(\delta/2T) \sqrt{TN}\cdot  \sqrt{\sum_{t=1}^T\min\Bigg\{ \sum_{i=1}^N\bar{x}_i^{t\top}A_{t-1}^{-1}\bar{x}_i^t, N\Bigg\}}\\
&\leq  \Ct\cdot K+2\rho_T(\delta/2T) \sqrt{TN}\cdot\sqrt{2dNs\log\Big(\frac{\lambda_0+B^2NT}{d}\Big)-2dNs\log(\lambda_0)}
\end{align*}
\endgroup
where we used the fact that $N\Ct\leq 2N\rho_T(\delta/2T)$ and Cauchy-Schwarz. Plugging in the definition of $\rho_T(\delta/2T) = \frac{2L}{\eta}\kappa \Ct \sqrt{2d\log T}\sqrt{\log(2dT/\delta)}$ finishes the proof.
\end{proof}

\section{Iterative Method}
\label{sec:iterative_method}
\begin{proof}[Proof of Proposition~\ref{prop:sgd}]
We begin by showing iterate contraction. Looking at the condition for potential gradient update, assuming that we have an upper bound as $\mathbb{E}[\|\beta_t-\beta^*\|_2]\leq d_t$ at iteration $t$, on the event that the distance to OPT satistifes $\|\beta^*-\beta_t\|_2\leq (1+\delta)\cdot d_t$, which happens with probability at least $1-\frac{1}{\delta}$ by Markov's inequality,
\begin{align*}
\mathbbm{1}\{\mu(x_t^{\top}\beta_t)+s_t\geq c\} &= \mathbbm{1}\left\{\mu(x_t^{\top}\beta^*)+s_t\geq c+\mu(x_t^{\top}\beta^*)-\mu(x_t^{\top}\beta_t)\right\}\\
&\geq \mathbbm{1}\left\{\mu(x_t^{\top}\beta^*)+s_t\geq c+L\|x_t\|_2\cdot \|\beta^*-\beta_t\|_2\right\}\\
&\geq \mathbbm{1}\left\{\mu(x_t^{\top}\beta^*)\geq c\right\}
\end{align*}
where we used that the exploration bonus $s_t = L\cdot (1+\delta)\cdot d_t\|x_t\|_2$ and the Lipschitz condition of $\mu'(\cdot)$. Now since $\mu(x_t^{\top}\beta^*)- \mathbb{E}[\mu(x^{\top}\beta^*)]$ is $CL\|\beta^*\|\sigma$-subgaussian for some numerical constant $C$, i.e., 
\[\mathbb{P}_x\left\{\mu(x_t^{\top}\beta^*)\geq c\right\}=\mathbb{P}\left\{\mu(x_t^{\top}\beta^*)\geq \mathbb{E}_x[\mu(x^{\top}\beta^*)]-\zeta\right\}\geq 1-e^{-\frac{\zeta^2}{2L^2\|\beta^*\|^2\sigma^2}}\, ,\]
therefore with probability at least $1-e^{-\frac{\zeta^2}{2L^2\|\beta^*\|^2\sigma^2}}-\delta^{-1}$, we accept $x_t$ and are presented with the corresponding response $y_t$. It remains to work out the update for $d_t$ such that $\mathbb{E}[\|\beta_t-\beta^*\|_2] \leq d_t$ holds at all iterations (so that we can set $s_t$ appropriately). 

For this, we have if $\left|y_t-\mu(x_t^{\top}\beta_t)\right| \leq \alpha+B$, thanks to the projection that maintains $\mu'(\cdot)\geq \gamma$ and the norm bound assumption on noise $\epsilon_t$,
\begin{align}
\label{eqn:intermediate}
\gamma|x_t^{\top}(\beta_t-\beta^*)|-B \leq  \left|\mu(x_t^{\top}\beta^*)-\mu(x_t^{\top}\beta_t)\right|-B \leq \left|y_t-\mu(x_t^{\top}\beta_t)\right| \leq \alpha+B\, ,
\end{align}
implying that we are already accurate enough on this sample as $|x_t^{\top}(\beta_t-\beta^*)|\leq \frac{\alpha+2B}{\gamma}$.
% Since $x_t^{\top}  (\beta_t-\beta^*)$ is zero-mean subgaussian with parameter at most $d_t\sigma$, we have
% \[\mathbb{P}_x\left(|x_t^{\top}  (\beta_t-\beta^*) |\leq \omega\cdot\gamma d_t^2 \right) \geq 1-e^{-\frac{\omega^2\gamma^2d_t^2}{2\sigma^2}}\, .\]

Otherwise if $\left|y_t-\mu(x_t^{\top}\beta_t)\right| > \alpha+B$, we have \begin{align*}
L|x_t^{\top}(\beta_t-\beta^*)|+B> \left|\mu(x_t^{\top}\beta^*)-\mu(x_t^{\top}\beta_t)\right|+B >\left|y_t-\mu(x_t^{\top}\beta_t)\right| > \alpha+B
\end{align*}
therefore $|x_t^{\top}(\beta_t-\beta^*)| > \frac{\alpha}{L}$,  % \omega\cdot\gamma d_t^2
and making a gradient update gives the contraction \[\mathbb{E}[\|\beta_{t+1}-\beta^*\|_2^2 \,\vert\, \beta_t] \leq \|\beta_t-\beta^*\|_2^2+\left[\eta^2\cdot \mu'(z)^2\|x_t\|_2^2-2\eta\cdot\mu'(z)\right](x_t^{\top}(\beta_t-\beta^*))^2\]
at step $t$. 
% This condition ensures that we make sufficient progress if we take a closer look at the term:
% \[ \|(\beta_t-\beta^*)_{\parallel}\|_2^2 = \frac{1}{\|x_t\|_2^4}\|x_t x_t^{\top}  (\beta_t-\beta^*)\|_2^2 = \frac{1}{\|x_t\|_2^2}\left(x_t^{\top}  (\beta_t-\beta^*)\right)^2>\frac{\omega^2\gamma^2d_t^4}{\|x_t\|_2^2}\, .\] 
%(we hope to show it's large enough so we make sufficient progress when there is a gradient update)
% \begin{align*}
%   \mathbb{P}\left(\|(\beta_t-\beta^*)_{\parallel}\|_2^2 \leq \omega^2 \|\beta_t-\beta^*\|_2^2 \right) 
%   & \leq \mathbb{P}\left(\frac{1}{\|x_t\|_2^4}\|x_tx_t^{\top}  (\beta_t-\beta^*)\|_2^2 \leq \omega^2\cdot d_t^2\right)\\
%   &= \mathbb{P}\left(\frac{1}{\|x_t\|_2}\left|x_t^{\top}  (\beta_t-\beta^*)\right| \leq \omega\cdot d_t\right)\\
% %   &\leq \mathbb{P}_x\left(|x_t^{\top}  (\beta_t-\beta^*) |\leq \frac{\omega\cdot d_t}{\delta\sigma\sqrt{d}}\right)+\mathbb{P}_x\left(\|x_t\|_2 \geq \delta\sigma\sqrt{d}\right)\\
%     &\leq \mathbb{P}\left(\|x_t\|_2 \geq \delta\sigma\sqrt{d}\right)\\
%   &\leq 2e^{-\frac{d\delta^2}{16}}
% \end{align*}
%
Taking stepsize $\eta = \frac{1}{L\|x_t\|_2^2} < \frac{1}{\mu'(z)\cdot\|x_t\|_2^2}$, we have the distance to OPT progress recursion using~\eqref{eqn:iterate_contract} as
\[\mathbb{E}[\|\beta_{t+1}-\beta^*\|^2\,\vert\, \beta_t] \leq \|\beta_t-\beta^*\|_2^2 - \frac{\alpha^2}{\|x_t\|_2^2 L^2}\, .\]
% \begin{align*}
% \mathbb{E}[\|\beta_{t+1}-\beta^*\|^2|\beta_t] &\leq (1-\eta\gamma \|x_t\|_2^2)^2 \cdot \|\beta_t-\beta^*\|_2^2+\left(1-\frac{\omega^2\gamma^2d_t^2}{\|x_t\|_2^2}\right)\|\beta_t-\beta^*\|_2^2
% \end{align*}
Using law of total expectation and putting everything together, at iteration $t$ of the algorithm, we have with probability at least $1-\rho := 1-e^{-\frac{\zeta^2}{2L^2\|\beta^*\|^2\sigma^2}}-\delta^{-1}$, either %(where we used $\|x_t\|_2 \leq \sqrt{d}\|x_t\|_{\infty}\leq \sqrt{d}\sigma$)
% \[\mathbb{E}[\|\beta_{t+1}-\beta^*\|^2] \leq \left(1-\frac{\gamma^2d_t^2}{\|x_t\|_2^2}\right) \cdot \mathbb{E}[\|\beta_{t}-\beta^*\|^2]\leq \left(1-\frac{\gamma^2d_t^2}{d\sigma^2}\right) \cdot \mathbb{E}[\|\beta_{t}-\beta^*\|^2]\, ,\]
\begin{equation}
\label{eqn:recursion}
\mathbb{E}[\|\beta_{t+1}-\beta^*\|^2] \leq \mathbb{E}[\|\beta_t-\beta^*\|_2^2] - \frac{\alpha^2}{\|x_t\|_2^2 L^2}\, ,
\end{equation}
or in light of~\eqref{eqn:intermediate},
\begin{align*}
\left|\mu(x_t^{\top}\beta^*)-\mu(x_t^{\top}\beta_t)\right| &\leq L\cdot |x_t^{\top}(\beta^*-\beta_t)| \leq \frac{L}{\gamma} (\alpha+2B)\, ,
\end{align*}
otherwise with the remaining probability $\rho$ the distance to OPT stays the same since we don't accept the sample. 
%In particular, this suggests if we initialize within a ball of radius $d_0$ around OPT, we never leave the ball by induction. This also says that we should update our exploration bonus $s_t$ as the suggested contraction as above during the execution of the algorithm.
%
Equation~\eqref{eqn:recursion} therefore prescribes that we should update $d_t$ as
\begin{align*}
\mathbb{E}\left[\|\beta_t-\beta^*\| \right] \leq \sqrt{\mathbb{E}\left[\|\beta_t-\beta^*\|^2\right]} \leq  \left(\mathbb{E}[\|\beta_{t-1}-\beta^*\|_2^2] - \frac{\alpha^2}{\|x_{t-1}\|_2^2 L^2}\right)^{1/2}\, .
% &\leq   \sqrt{\mathbb{E}\left[\|\beta_t-\beta^*\|^2 \right]}\\
% &\leq   \sqrt{1-\frac{\gamma^2d_t^2}{d\sigma^2}} \cdot \|\beta_{t-1}-\beta^*\|\\
% &\leq  \left(1-\frac{\gamma^2d_t^2}{2d\sigma^2}\right)\|\beta_{t-1}-\beta^*\|
\end{align*}
This concludes the first part of the claim.
% where we used $\sqrt{1-x} \leq 1-\frac{x}{2}$. 
%
% \subsection{Misclassification Error}
% In this part, we give the bound on making a mis-classification error when the exploration bonus is set as in Algorithm (CITE). 
Turning to misclassification error, since again by sub-gaussianity,
\[\mathbb{P}_x\{\mu(x_t^{\top}\beta^*)\geq c\} = \mathbb{P}\left\{\mu(x_t^{\top}\beta^*)\geq \mathbb{E}_x[\mu(x^{\top}\beta^*)]-\zeta\right\}\geq 1-e^{-\frac{\zeta^2}{2L^2\|\beta^*\|^2\sigma^2}}\, ,\]
we focus on this side of the error. In this case, we don't make a mistake w.r.t the oracle predictor $\mathbbm{1}\{\mu(x_t^{\top}\beta^*) \geq c\}$ at step $t$ if $\mu(x_t^{\top}\beta_t)+s_t \geq c$, in which case we predict ``accept" and is revealed $y_t$.
To bound the probability, we observe that
\begin{align*}
 \mathbb{P}(\mu(x_t^{\top}\beta_t) +s_t\geq c) &= \mathbb{P}(\mu(x_t^{\top}\beta_t) +\mu(x_t^{\top}\beta^*)-\mu(x_t^{\top}\beta^*)+s_t\geq c)  \\
 &\geq \mathbb{P}(s_t \geq \mu(x_t^{\top}\beta^*)-\mu(x_t^{\top}\beta_t))\\
 &= \mathbb{P}(L\cdot(1+\delta)\cdot d_t\|x_t\|_2\geq \mu(x_t^{\top}\beta^*)-\mu(x_t^{\top}\beta_t))\\
&\geq \mathbb{P}(\|\beta_t-\beta^*\|_2 \leq (1+\delta)\cdot d_t)\\
&\geq 1-\frac{\mathbb{E}[\|\beta_t-\beta^*\|_2]}{\delta\cdot d_t} \geq 1-\frac{1}{\delta}\, .
\end{align*}
% \[\mu(x_t^{\top}\beta_t)-\mu(x_t^{\top}\beta^*) \leq c-\mu(x_t^{\top}\beta^*)\,.\] 
% Now for the RHS (this is time-step-independent) we know from sub-gaussianity that
% \begin{align*}
% \mathbb{P}\left\{\mathbb{E}_x[\mu(x^{\top}\beta^*)]+\zeta -\mu(x_t^{\top}\beta^*)\geq \frac{1}{2}\zeta\right\}
% \geq 1-e^{-\frac{\zeta^2}{8L^2\|\beta^*\|^2\sigma^2}}\, .
% \end{align*}
%
%
Consequently, picking $\delta^{-1} =\rho-e^{-\frac{\zeta^2}{2L^2\|\beta^*\|^2\sigma^2}}$ in setting $s_t$ and assuming $\zeta \geq \sqrt{2L\|\beta^*\|^2\sigma^2\log(\rho^{-1})}$, 
\begin{align*}
    \mathbb{P}(\text{making a mistake at step $t$}) &\leq \mathbb{P}(\mu(x_t^{\top}\beta^*)<c)+\mathbb{P}(\mu(x_t^{\top}\beta_t)+s_t<c \,\vert\, \mu(x_t^{\top}\beta^*) \geq c)\\
    &\leq e^{-\frac{\zeta^2}{2L^2\|\beta^*\|^2\sigma^2}} + \frac{1}{\delta} \leq \rho
\end{align*}
for any $\rho \in (0,1)$.
\end{proof}

% \qijia{In general, it's quite likely that you can have smallish misclassifaction error while having $\beta_t$ being quite far from $\beta^*$ (think of very high cutoff).}

% Now at step $t$, we note that we don't make a mistake w.r.t the oracle predictor $\mu(x_t^{\top}\beta^*)$ for a cutoff $c$ when 
% \[\left|\mu(x_t^{\top}\beta^*)-\mu(x_t^{\top}\beta_t)\right| \leq \left|\mu(x_t^{\top}\beta^*)-c\right|\, .\]

% A natural notion of regret as an upper bound on number of mistakes in this case is
% \[R_T := \frac{1}{T}\sum_{t=1}^T \mathbb{P}\left\{\left|\mu(x_t^{\top}\beta^*)-\mu(x_t^{\top}\beta_t)\right| > \left|\mu(x_t^{\top}\beta^*)-c\right|\right\}\, .\]

% %where since $x_t$ is independent from $\beta_{t-1}$ 
% and since 
% \[\mathbb{P}\left(\|x_t\|_2 \leq \delta\sigma\sqrt{d}\right)
%   \geq 1-2e^{-\frac{d\delta^2}{16}}\]
% where we used the well-known tail norm bound for subgaussian random vector [CITE]. 

%Union bounding over the two events gives ... regret \[R_T  \leq \frac{1}{T}\sum_t \]
%if we initialize with $d_0 \leq $ ...

%\newpage
\section{Additional Experiment Results}
We include additional tables and plots in this section to further support our findings, starting on next page.
\label{sec:appendix_simulation}

\begin{table}[t]
\setlength{\tabcolsep}{8pt}
\centering
\begin{tabular}{c|c|c|c|c|c|c|c|c}
\hline 
   Dataset & c & greedy & $\epsilon$-grdy & os-$\epsilon$-grdy   & noise           & os-noise     & margin   & ours        \\ \hline\hline
 \multirow{10}{*}{Adult} & 50\% & 239.45 & 236.34 & 211.74 & 230.77 & 165.77 & 162.31 & {\bf 144.92} \\ \cline{2-9}
  & 55\% & 175.88 & 175.46 & 170.21 & 175.26 & 143.64 & 135.9 & {\bf 118.7} \\ \cline{2-9}
  & 60\% & 140.18 & 138.37 & 138.13 & 138.53 & 126.76 & 125.37 & {\bf 114.53} \\ \cline{2-9}
  & 65\% & 129.29 & 128.39 & 128.26 & 126.8 & 123.48 & 120.94 & {\bf 115.62} \\ \cline{2-9}
  & 70\% & 134.74 & 134.18 & 133.8 & 131.66 & 132.39 & 132.67 & {\bf 129.81} \\ \cline{2-9}
  & 75\% & 145.99 & 145.14 & 146.09 & {\bf 144.38} & 146.26 & 145.28 & 145.86 \\ \cline{2-9}
  & 80\% & 188.48 & 186.69 & 186.03 & 185.47 & {\bf 184.68} & 185.44 & 186.2 \\ \cline{2-9}
  & 85\% &  246.93 & 244.71 & 244.93 & 243.14 & 243.77 & {\bf 242.94} & 243.16 \\ \cline{2-9}
  & 90\% & 318.33 & 295.72 & {\bf 290.51} & 279.49 & 293.3 & 294.4 & 293.14 \\ \cline{2-9}
  & 95\% &  179.24 & 146.21 & 158.73 & {\bf 131.2} & 148.85 & 131.75 & 152.95 \\ \hline
 \multirow{10}{*}{Bank} & 50\% & 164.23 & 162.67 & 117.86 & 136.0 & 88.49 & 86.26 & {\bf 74.64 } \\ \cline{2-9}
 & 55\% & 142.01 & 138.29 & 107.39 & 125.15 & 83.59 & 82.46 & {\bf 71.68} \\ \cline{2-9}
 & 60\% & 141.04 & 139.42 & 110.72 & 131.61 & 90.86 & 89.29 & {\bf 81.02} \\ \cline{2-9}
 & 65\% & 146.56 & 140.44 & 121.98 & 135.96 & 100.98 & 96.59 & {\bf 94.46 } \\ \cline{2-9}
 & 70\% & 207.6 & 197.0 & 185.9 & 198.66 & 153.3 & 150.75 & {\bf 137.24 }\\ \cline{2-9}
 & 75\% & 166.72 & 166.08 & 166.81 & 162.88 & 153.38 & 150.07 & {\bf 142.5} \\ \cline{2-9}
 & 80\% & 148.93 & 147.06 & 148.75 & 148.18 & 142.68 & 137.37 & {\bf 134.17} \\ \cline{2-9}
 & 85\% & 145.63 & 125.29 & 122.99 & 122.59 & 122.28 & 130.96 & {\bf 119.19} \\ \cline{2-9}
 & 90\% & 104.98 & 102.89 & 104.59 & 103.0 & 102.1 & {\bf 101.1} & 104.19 \\ \cline{2-9}
 & 95\% &  119.77 & 119.06 & {\bf 116.17} & 119.42 & 119.48 & 119.63 & 119.73 \\ \hline
 \multirow{10}{*}{COMPAS} & 50\% & 41.56 & 36.67 & 36.93 & 36.93 & 28.09 & 28.12 & {\bf 26.01} \\  \cline{2-9}
 & 55\% & 41.71 & 38.18 & 39.22 & 37.72 & 33.47 & 31.85 & {\bf 31.17} \\  \cline{2-9}
 & 60\% & 44.83 & 44.02 & 42.78 & 42.2 & 34.84 & 36.56 & {\bf 34.48} \\  \cline{2-9}
 & 65\% & 47.33 & 40.04 & 40.06 & 38.23 & 35.44 & 33.7 & {\bf 32.52} \\  \cline{2-9}
 & 70\% & 41.66 & 39.16 & 39.61 & 39.87 & 38.03 & 36.98 & {\bf 34.07} \\  \cline{2-9}
 & 75\% & 46.14 & 40.49 & 41.12 & 37.84 & 35.11 & 34.27 & {\bf 33.89} \\  \cline{2-9}
 & 80\% &  45.8 & 45.25 & 44.88 & 44.58 & 43.97 & 41.11 & {\bf 40.63} \\  \cline{2-9}
 & 85\% & 58.59 & 54.62 & 50.54 & 50.52 & 43.56 & 46.24 & {\bf 41.93} \\  \cline{2-9}
 & 90\% & 33.42 & 32.1 & 33.66 & {\bf 28.79} & 31.71 & 31.91 & 30.47 \\  \cline{2-9}
 & 95\% & 20.38 & 20.34 & 20.37 & 20.31 & 20.38 & {\bf 19.68} & 20.35 \\ \hline
 \multirow{10}{*}{Crime}  & 50\%  & 15.77 & 15.77 & 15.5 & 15.66 & 14.93 & 14.73 & {\bf 13.95}       \\  \cline{2-9}
  & 55\%  & 14.64 & 14.52 & 14.64 & 14.27 & 14.46 & 14.46 & {\bf  14.07}      \\  \cline{2-9}
 & 60\% & 17.44 & 17.42 & 17.31 & 17.12 & 16.99 & 16.35 & {\bf 15.95} \\  \cline{2-9}
 & 65\% & 18.64 & 18.59 & 18.72 & 18.27 & 18.52 & 18.53 & {\bf 18.15} \\  \cline{2-9}
 & 70\% & 22.0 & 21.75 & 21.99 & 20.33 & 20.63 & 20.1 & {\bf 19.19} \\  \cline{2-9}
 & 75\% & 21.87 & 21.87 & 21.74 & 21.33 & 21.05 & 21.4 & {\bf 20.75} \\  \cline{2-9}
 & 80\% & 23.03 & 22.94 & 23.38 & 22.46 & 22.61 & 22.58 & {\bf 21.87} \\  \cline{2-9}
 & 85\% &  23.75 & 23.46 & 23.82 & 23.75 & 23.65 & 23.51 & {\bf 22.68} \\  \cline{2-9}
 & 90\% & 22.43 & 22.19 & 22.34 & 21.65 & 22.13 & 21.3 & {\bf 20.88} \\  \cline{2-9}
 & 95\% & 12.34 & 12.34 & 12.43 & 12.34 & {\bf 12.21} & 12.34 & 12.34 \\ \hline
 \multirow{10}{*}{German} & 50\% & 14.7 & 14.51 & 14.12 & 13.62 & 11.12 & 10.52 & {\bf 9.63} \\  \cline{2-9}
 & 55\% & 12.43 & 12.3 & 12.24 & 12.42 & 11.06 & 10.98 & {\bf 9.43} \\  \cline{2-9}
 & 60\% & 15.16 & 14.48 & 14.2 & 14.09 & 13.83 & 13.32 & {\bf 11.68} \\  \cline{2-9}
 & 65\% &  17.02 & 16.39 & 16.52 & 15.82 & 13.75 & 13.86 & {\bf 12.48} \\  \cline{2-9}
 & 70\% & 15.89 & 15.53 & 15.93 & 15.41 & 14.09 & 14.52 & {\bf 13.07} \\  \cline{2-9}
 & 75\% &  15.44 & 15.26 & 15.13 & 14.86 & 14.49 & 14.95 & {\bf 14.2} \\ \cline{2-9}
 & 80\% & 12.8 & 12.69 & 12.87 & 12.61 & 12.68 & 12.63 & {\bf 12.45} \\  \cline{2-9}
 & 85\% & 11.55 & 11.45 & 11.23 & 11.38 & 11.27 & 11.23 & {\bf 10.98} \\  \cline{2-9}
 & 90\% & 10.09 & 9.96 & 10.14 & {\bf 9.07} & 9.84 & 9.97 & 9.93 \\  \cline{2-9}
 & 95\% & 8.23 & 8.13 & 8.23 & {\bf 7.59} & 7.97 & 8.18 & 8.22 \\ \hline
\end{tabular}
%\vspace{0.2cm}
\caption{\label{tab:linear_table_all_1}  {\bf  Linear Regression Results}. }
\end{table}

\begin{table}[t]
\setlength{\tabcolsep}{8pt}
\centering
\begin{tabular}{c|c|c|c|c|c|c|c|c}
\hline
   Dataset & c &greedy & $\epsilon$-grdy & os-$\epsilon$-grdy   & noise           & os-noise     & margin   & ours        \\ \hline\hline
 \multirow{10}{*}{Blood} & 50\% & 2.06 & 2.06 & 2.06 & 2.06 & 1.92 & 1.72 & {\bf 1.52} \\ \cline{2-9}
 & 55\% &  1.4 & 1.4 & 1.4 & {\bf 1.36} & 1.39 & 1.39 & 1.38 \\ \cline{2-9}
 & 60\% & 3.63 & 2.72 & 3.11 & 1.94 & 1.91 & 1.96 & {\bf 1.87} \\ \cline{2-9}
 & 65\% & 3.28 & 2.74 & 2.07 & 2.81 & 2.02 & 1.69 & {\bf 1.59} \\ \cline{2-9}
 & 70\% & 3.7 & 2.78 & 3.04 & {\bf 2.38} & 3.13 & 3.06 & 2.65 \\ \cline{2-9}
 & 75\% & 5.03 & 3.91 & 4.16 & 3.29 & 4.08 & 3.99 & {\bf 3.13} \\  \cline{2-9}
 & 80\% & 4.16 & 3.32 & 4.12 & 3.07 & {\bf 3.06} & 3.92 & 3.58 \\ \cline{2-9}
 & 85\% & 4.1 & 3.73 & 3.58 & {\bf 3.28} & 3.98 & 4.05 & 3.67 \\ \cline{2-9}
 & 90\% & 5.09 & 4.58 & 5.11 & {\bf 3.97} & 4.26 & 4.51 & 4.66 \\ \cline{2-9}
 & 95\% &  2.64 & 2.59 & 2.68 & 2.61 & 2.57 & 2.56 & {\bf 2.55} \\ \hline
 \multirow{10}{*}{Diabetes} & 50\% & 4.17 & 4.16 & 4.23 & 3.94 & 3.81 & 3.95 & {\bf 3.61} \\ \cline{2-9}
 & 55\% & 4.93 & 4.93 & 4.97 & 4.88 & 4.79 & 4.9 & {\bf 4.74} \\ \cline{2-9}
 & 60\% &  6.01 & 6.01 & 5.92 & 5.83 & 5.75 & 5.97 & {\bf 5.65} \\ \cline{2-9}
 & 65\% & 5.45 & 5.45 & 5.48 & 5.29 & 5.32 & 5.3 & {\bf 5.25} \\ \cline{2-9}
 & 70\% & 6.05 & 5.56 & 6.14 & 6.05 & 5.6 & 5.39 & {\bf 5.33} \\ \cline{2-9}
 & 75\% & 7.61 & 7.57 & 6.74 & 6.69 & 6.21 & 6.35 & {\bf 5.5} \\ \cline{2-9}
 & 80\% & 8.18 & 7.9 & 8.24 & 7.64 & 8.01 & 7.06 & {\bf 6.65} \\ \cline{2-9}
 & 85\% &  6.84 & 6.84 & 6.84 & 6.84 & 6.75 & 6.69 & {\bf 6.64} \\ \cline{2-9}
 & 90\% & 5.78 & 5.73 & 5.86 & 5.47 & 5.65 & 5.53 & {\bf 5.51} \\ \cline{2-9}
 & 95\% & 4.52 & 4.52 & 4.56 & 4.36 & 4.37 & 4.29 & {\bf 4.27} \\ \hline
 \multirow{10}{*}{EEG Eye} & 50\% & 256.47 & 200.04 & 175.8 & 173.52 & 106.26 &{\bf 96.85} & 119.7 \\ \cline{2-9}
 & 55\% &  227.08 & 191.27 & 169.19 & 177.03 & 118.03 & {\bf 109.69} & 128.09 \\ \cline{2-9}
 & 60\% & 196.1 & 169.52 & 163.42 & 155.87 & 121.5 & 121.09 & {\bf 119.67} \\ \cline{2-9}
 & 65\% & 162.28 & 159.8 & 154.73 & 148.5 & 133.16 & 130.01 & {\bf 129.27} \\ \cline{2-9}
 & 70\% & 175.71 & 167.94 & 168.73 & 157.68 & 167.52 & 160.76 & {\bf 155.79} \\ \cline{2-9}
 & 75\% &  157.93 & 147.06 & 154.61 & 146.3 & {\bf 123.48} & 124.6 & 136.1\\ \cline{2-9}
 & 80\% &  164.19 & 140.15 & 139.47 & 133.81 & 142.57 & {\bf 125.39} & 149.71 \\ \cline{2-9}
 & 85\% & 143.94 & 125.09 & 118.29 & {\bf 115.81} & 117.24 & 131.08 & 136.51 \\\cline{2-9}
 & 90\% & 121.78 & 116.0 & 121.3 & {\bf 104.72} & 115.71 & 115.27 & 117.57 \\ \cline{2-9}
 & 95\% & 149.06 & {\bf 142.67} & 145.99 & 139.5 & 151.72 & 148.77 & 149.09 \\ \hline
 \multirow{10}{*}{Australian} & 50\% & 3.74 & 3.74 & 3.77 & 3.63 & 3.0 & 2.79 & {\bf 2.65} \\ \cline{2-9}
 & 55\% & 3.38 & 3.38 & 3.38 & 3.26 & 2.96 & 3.19 & {\bf 2.69} \\ \cline{2-9}
 & 60\% & 5.0 & 4.97 & 5.0 & 4.33 & {\bf 3.73} & 3.99 & 3.75 \\ \cline{2-9} 
 & 65\% & 4.69 & 4.57 & 4.69 & 4.3 & 3.88 & {\bf 3.84} & 3.9 \\ \cline{2-9}
 & 70\% & 6.77 & 6.77 & 6.77 & 6.66 & 5.09 & 5.26 & {\bf 4.65} \\ \cline{2-9}
 & 75\% & 5.78 & 5.77 & 5.78 & 5.57 & 4.99 & 4.8 & {\bf 4.77} \\ \cline{2-9}
 & 80\% & 5.43 & 5.43 & 5.43 & 5.21 & 5.27 & 5.25 & {\bf 4.85} \\ \cline{2-9}
 & 85\% & 5.09 & 5.09 & 5.09 & {\bf 5.01} & 5.03 & 4.98 & 5.06 \\ \cline{2-9}
 & 90\% & 4.14 & 4.14 & 4.14 & {\bf 3.93} & 4.08 & 4.07 & 4.11 \\ \cline{2-9}
 & 95\% & 2.15 & 2.14 & 2.27 &{ \bf 2.08} & 2.15 & 2.15 & 2.15 \\ \hline
 \multirow{10}{*}{Churn} & 50\% & 46.98 & 43.65 & 30.65 & 36.64 & 21.24 & 18.83 & {\bf 14.89} \\ \cline{2-9}
 & 55\% & 57.49 & 51.72 & 47.36 & 46.52 & 30.75 & {\bf 23.76} & 24.39 \\ \cline{2-9}
 & 60\% &  61.44 & 56.94 & 50.43 & 55.59 & 35.65 & 32.82 & {\bf 29.55} \\ \cline{2-9}
 & 65\% &  40.83 & 38.89 & 37.29 & 37.67 & 29.44 & 29.33 & {\bf 26.14} \\ \cline{2-9}
 & 70\% & 49.99 & 47.84 & 47.91 & 49.89 & 41.18 & 36.17 & {\bf 35.27} \\ \cline{2-9}
 & 75\% & 58.96 & 56.91 & 58.34 & 56.99 & 55.42 & 53.88 & {\bf 49.48} \\ \cline{2-9}
 & 80\% & 52.66 & 50.43 & 51.85 & 51.23 & 48.62 & 48.74 & {\bf 48.25} \\ \cline{2-9}
 & 85\% & 52.41 & 50.66 & 49.02 & {\bf 45.95} & 50.76 & 49.67 & 49.89 \\ \cline{2-9}
 & 90\% & 60.33 & {\bf 59.5 } & 60.37 & 59.98 & 59.78 & 59.82 & 59.97 \\ \cline{2-9}
 & 95\% & 56.36 & 54.33 & 54.7 & {\bf 53.09} & 54.29 & 55.91 & 56.39 \\ \hline
\end{tabular}
%\vspace{0.2cm}
\caption{\label{tab:linear_table_all_2}  {\bf  Linear Regression Results (continued)}. }
\end{table}

% \begin{figure}[t]
%  \begin{center}
%    \includegraphics[width=0.35\textwidth]{figures/linear_adult_all_0_5.png}
%    \includegraphics[width=0.37\textwidth]{figures/linear_adult_all_0_7.png}
%    \includegraphics[width=0.35\textwidth]{figures/linear_bank_all_0_5.png}
%    \includegraphics[width=0.35\textwidth]{figures/linear_bank_all_0_7.png}
%    \includegraphics[width=0.35\textwidth]{figures/linear_compas_all_0_5.png}
%    \includegraphics[width=0.37\textwidth]{figures/linear_compas_all_0_7.png}
%    \includegraphics[width=0.37\textwidth]{figures/linear_crime_all_0_5.png}
%    \includegraphics[width=0.37\textwidth]{figures/linear_crime_all_0_7.png}
%    \includegraphics[width=0.37\textwidth]{figures/linear_german_all_0_5.png}
%    \includegraphics[width=0.37\textwidth]{figures/linear_german_all_0_7.png}
%    \end{center}
%  \caption{{\bf Linear Regression plots}. }
%	\label{fig:linear_plots_all_1}
%\end{figure}

 \begin{figure}[t]
  \begin{center}
    \includegraphics[width=0.8\textwidth]{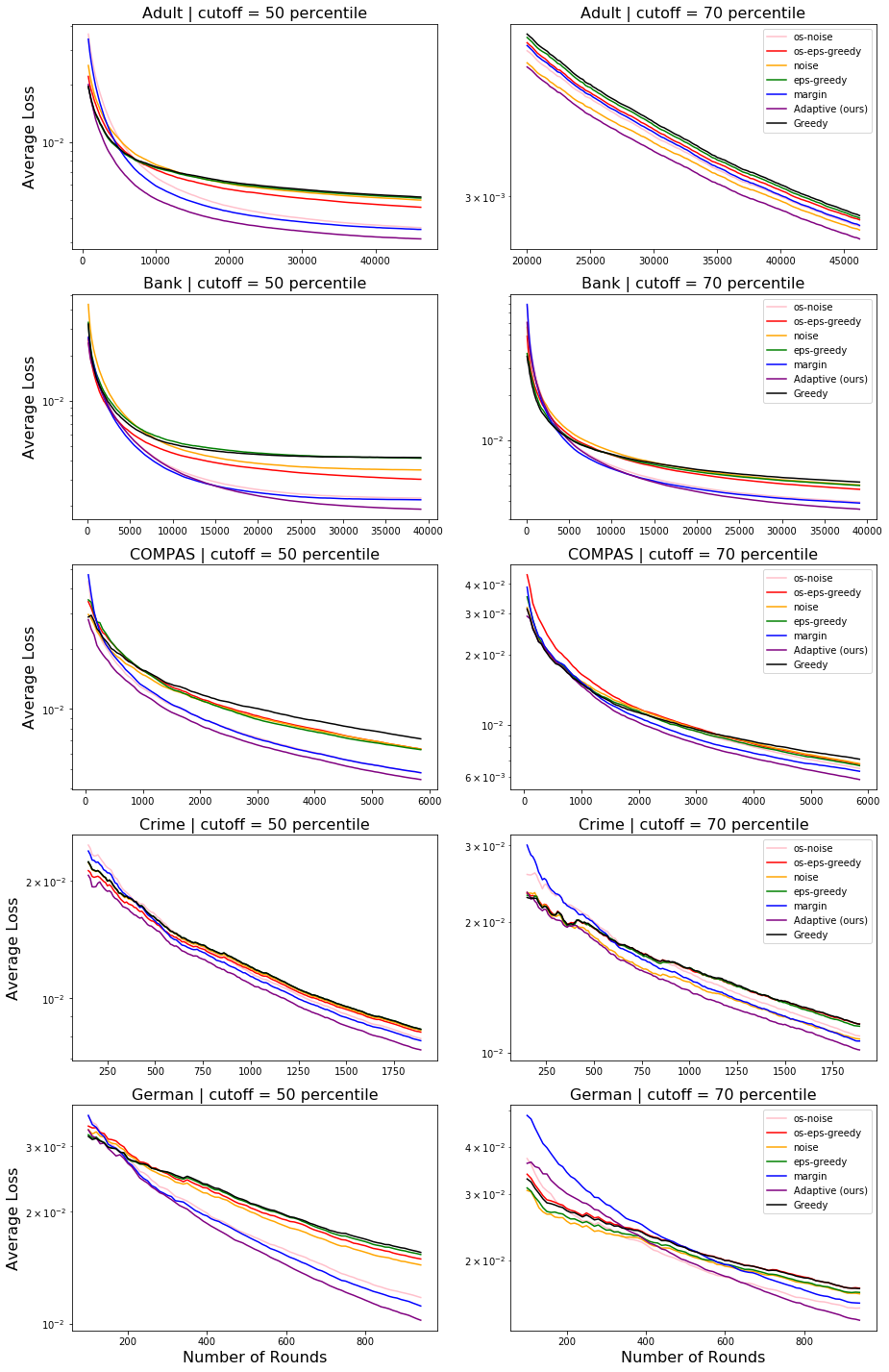}
    \end{center}
  \caption{{\bf Linear Regression plots}. }
	\label{fig:linear_plots_all_1}
\end{figure}

 \begin{figure}[t]
  \begin{center}
    \includegraphics[width=0.8\textwidth]{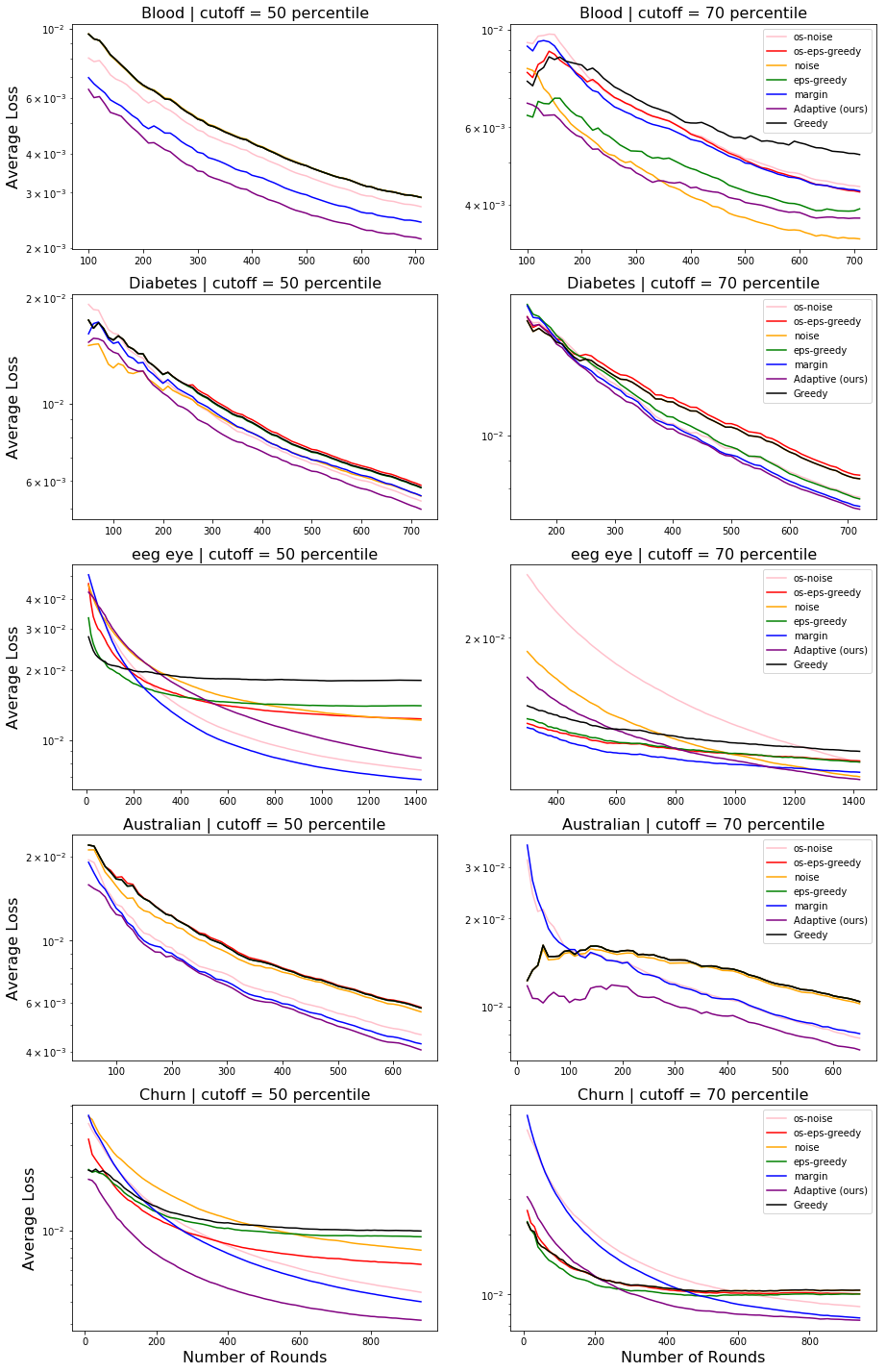}
    \end{center}
  \caption{{\bf Linear Regression plots (continued)}. }
	\label{fig:linear_plots_all_2}
\end{figure}

% \begin{figure}[t]
%  \begin{center}
%    \includegraphics[width=0.35\textwidth]{figures/linear_blood_all_0_5.png}
 %   \includegraphics[width=0.35\textwidth]{figures/linear_blood_all_0_7.png}
 %   \includegraphics[width=0.35\textwidth]{figures/linear_diabetes_all_0_5.png}
%    \includegraphics[width=0.35\textwidth]{figures/linear_diabetes_all_0_7.png}
%    \includegraphics[width=0.35\textwidth]{figures/linear_eeg_all_0_5.png}
%    \includegraphics[width=0.35\textwidth]{figures/linear_eeg_all_0_7.png}
%    \includegraphics[width=0.35\textwidth]{figures/linear_australian_all_0_5.png}
%%    \includegraphics[width=0.35\textwidth]{figures/linear_australian_all_0_7.png}
%    \includegraphics[width=0.35\textwidth]{figures/linear_churn_all_0_5.png}
%    \includegraphics[width=0.35\textwidth]{figures/linear_churn_all_0_7.png}
%    \end{center}
%  \caption{{ \bf Linear Regression plots (continued)}. }
%	\label{fig:linear_plots_all_2}
%\end{figure}

\begin{table}[t]
\setlength{\tabcolsep}{8pt}
\centering
\begin{tabular}{c|c|c|c|c|c|c|c|c}
\hline
   Dataset & c &greedy & $\epsilon$-grdy & os-$\epsilon$-grdy   & noise           & os-noise     & margin   & ours        \\ \hline \hline
   \multirow{10}{*}{Adult} & 50\% & 43.48 & 43.55 & 43.48 & 43.35 & 43.41 & 43.38 & {\bf 42.63} \\ \cline{2-9}
   & 55\% & 59.78 & 59.9 & 59.79 & 59.61 & {\bf 59.58} & 59.64 &  59.61 \\ \cline{2-9}
   & 60\% & 77.25 & 77.44 & 77.26 & 76.98 & {\bf 76.48} & 76.58 & 77.05 \\ \cline{2-9}
   & 65\% & 95.62 & 95.86 & 95.76 & 95.43 & 94.95 & 94.75 & {\bf 93.13} \\ \cline{2-9}
   & 70\% & 102.86 & 102.86 & 102.9 & 102.6 & 102.81 & 102.47 & {\bf 100.06} \\ \cline{2-9}
   & 75\% & 123.34 & 123.55 & 123.11 & 120.7 & 120.08 & 119.94 & {\bf 118.38} \\ \cline{2-9}
   & 80\% & 117.39 & 117.44 & 117.56 & 115.95 & 116.39 & 115.5 & {\bf 112.05} \\ \cline{2-9}
   & 85\% & 94.83 & 94.94 & 94.83 & 94.71 & 93.4 & 92.53 & {\bf 90.86} \\ \cline{2-9}
   & 90\% & 77.83 & 78.05 & 78.36 & 77.12 & 76.61 & 77.24 & {\bf 75.22} \\ \cline{2-9}
   & 95\% & 32.12 & 32.48 & 32.8 & 31.86 & 31.14 & 31.47 & {\bf 30.81} \\ \hline
    \multirow{10}{*}{Bank} & 50\% & 23.22 & 23.26 & {\bf 23.18} & 23.3 & 23.33 & 23.2 & 23.23 \\ \cline{2-9}
   & 55\% & 34.55 & 34.55 & 34.53 & {\bf 34.48} & 34.6 & 34.52 & 34.58\\ \cline{2-9}
   & 60\% & 49.57 & 49.56 & 49.58 & 49.52 & 49.86 & {\bf 49.5} & 49.56\\ \cline{2-9}
   & 65\% & 71.09 & {\bf 70.76} & 71.05 & 71.16 & 71.33 & 71.01 &  71.01 \\ \cline{2-9}
   & 70\% & 85.72 & 85.94 & 85.67 & 85.51 & {\bf 85.26} & 85.27 & 85.75 \\ \cline{2-9}
   & 75\% & 115.67 & 115.52 & 115.71 & {\bf 114.63} & 115.42 & 115.19 & 115.79 \\ \cline{2-9}
   & 80\% & 149.07 & 148.31 & 149.12 & 147.86 & 147.34 & 148.19 & {\bf 147.1} \\ \cline{2-9}
   & 85\% & 162.05 & 161.33 & 162.16 & 159.92 & 159.59 & 160.8 & {\bf 158.34} \\ \cline{2-9}
   & 90\% & 135.85 & 135.48 & 135.76 & 134.89 & 132.68 & 131.4 & {\bf 121.53} \\ \cline{2-9}
   & 95\% & 98.26 & 98.07 & 98.15 & 96.57 & 97.01 & 96.52 & {\bf 89.9} \\ \hline
    \multirow{10}{*}{COMPAS} & 50\% & 44.47 & 43.88 & 44.15 & 43.07 & 42.11 & 42.64 & {\bf 40.34} \\ \cline{2-9}
    & 55\% & 47.87 & 47.45 & 47.85 & 47.19 & 46.98 & 46.61 & {\bf 45.51} \\ \cline{2-9}
    & 60\% & 44.38 & 43.91 & 44.36 & 43.95 & 41.94 & 41.62 & {\bf 40.62} \\ \cline{2-9}
    & 65\% & 42.54 & 41.98 & 41.59 & 40.69 & 39.94 & 40.58 & {\bf  39.63} \\ \cline{2-9}
    & 70\% & 43.7 & 43.59 & {\bf 43.41} & 43.66 & 43.83 & 43.7 & 43.7 \\ \cline{2-9}
    & 75\% & 35.85 & 35.8 & 35.17 & {\bf 33.77} & 35.51 & 35.51 & 35.07 \\ \cline{2-9}
    & 80\% & 38.79 & 38.74 & 38.85 & {\bf 36.71} & 38.29 & 37.16 & 36.79 \\ \cline{2-9}
    & 85\% & 27.51 & 27.51 & 27.47 & 27.28 & {\bf 27.14 } & 27.19 & 27.28 \\ \cline{2-9}
    & 90\% & 21.98 & 21.98 & 21.98 & 21.55 & {\bf 20.16} & 20.69 & 21.04 \\ \cline{2-9}
    & 95\% & 16.72 & {\bf 15.93} & 16.46 & 16.63 & 16.63 & 16.71 & 16.24 \\ \hline
    \multirow{10}{*}{Crime} & 50\% & 11.04 & 10.83 & 11.04 & 10.85 & 10.33 & 10.44 & {\bf 9.42} \\ \cline{2-9}
    & 55\% & 12.85 & 12.63 & 12.85 & 12.14 & 12.36 & 12.14 & {\bf 11.45} \\ \cline{2-9}
    & 60\% & 16.75 & 16.75 & 16.79 & 16.81 & 16.21 & 15.77 & {\bf 15.04} \\ \cline{2-9}
    & 65\% & 26.51 & 26.51 & 26.01 & 25.21 & 23.3 & 23.27 & {\bf 22.0} \\ \cline{2-9}
    & 70\% & 26.05 & 25.93 & 26.13 & 25.94 & 25.84 & 25.55 & {\bf 24.46} \\ \cline{2-9}
    & 75\% & 29.17 & 28.65 & 29.22 & 27.99 & 27.92 & 27.06 & {\bf 26.34} \\ \cline{2-9}
    & 80\% & 31.01 & 31.01 & 31.01 & 29.68 & 30.6 & 30.4 & {\bf 30.08} \\ \cline{2-9}
    & 85\% &  22.33 & 22.33 & 22.33 & 22.11 & 21.82 & 21.33 & {\bf 21.3} \\ \cline{2-9}
    & 90\% & 13.3 & 13.3 & 13.3 & {\bf 12.62} & 12.74 & 12.88 & 12.67 \\ \cline{2-9}
    & 95\% &  2.96 & 2.96 & {\bf 2.95} & 2.96 & 3.03 & 2.96 & 2.96 \\ \hline
    \multirow{10}{*}{German} & 50\% & 35.71 & 35.21 & 33.55 & 33.35 & 24.19 & 23.19 & {\bf 20.33} \\ \cline{2-9}
    & 55\% & 37.24 & 34.7 & 35.09 & 37.01 & 29.42 & 26.65 & {\bf 23.74} \\ \cline{2-9}
    & 60\% & 42.12 & 39.95 & 39.12 & 37.53 & 31.77 & 29.35 & {\bf 25.19} \\ \cline{2-9}
    & 65\% & 35.27 & 35.14 & 34.63 & 33.39 & 31.23 & 30.55 & {\bf 28.16} \\ \cline{2-9}
    & 70\% & 42.55 & 41.14 & 42.18 & 40.98 & 40.64 & 40.3 & {\bf 37.12} \\ \cline{2-9}
    & 75\% & 31.49 & 31.41 & 31.49 & 31.26 & 31.02 & 30.77 & {\bf 29.9} \\ \cline{2-9}
    & 80\% & 31.83 & 31.67 & 31.83 & 30.76 & 30.0 & 29.94 & {\bf 29.11} \\ \cline{2-9}
    & 85\% & 29.61 & 29.61 & 29.61 & {\bf 29.08} & 29.26 & 29.31 & 29.45 \\ \cline{2-9}
    & 90\% & 24.73 & 24.73 & {\bf 24.54} & 24.41 & 24.61 & 24.62 & 24.7 \\ \cline{2-9}
    & 95\% & 18.19 & 18.19 & 18.19 & 17.78 & 17.88 & 17.88 & {\bf 17.7} \\ \hline
\end{tabular}
\vspace{0.2cm}
\caption{\label{tab:linear_table_all}  {\bf  Logistic Regression Results}. }
\end{table}

\begin{table}[t]
\setlength{\tabcolsep}{8pt}
\centering
\begin{tabular}{c|c|c|c|c|c|c|c|c}
\hline
   Dataset & c &greedy & $\epsilon$-grdy & os-$\epsilon$-grdy   & noise           & os-noise     & margin   & ours        \\ \hline \hline
 \multirow{10}{*}{Blood} & 50\% & 5.05 & 5.05 & 4.87 & 4.83 & 4.71 & 4.53 & {\bf 4.24} \\ \cline{2-9}
 & 55\% &  6.44 & 6.44 & 6.16 & 6.46 & 6.17 & 6.15 & {\bf 5.91} \\ \cline{2-9}
 & 60\% & 6.97 & 6.97 & 6.97 & 6.87 & 6.19 & 5.5 & {\bf 5.45} \\ \cline{2-9}
 & 65\% & 11.34 & 11.34 & 11.34 & 11.59 & 9.02 & 8.72 & {\bf 8.18} \\ \cline{2-9}
 & 70\% & 13.04 & 13.04 & 13.03 & 13.04 & 10.84 & 12.14 & {\bf 9.69} \\ \cline{2-9}
 & 75\% & 6.13 & 6.13 & 6.13 & 5.76 & 5.44 & 4.54 & {\bf 4.31} \\ \cline{2-9}
 & 80\% & 8.92 & 8.88 & 8.92 & 8.77 & 9.05 & 8.79 & {\bf 8.03} \\ \cline{2-9}
 & 85\% &  2.63 & 2.63 & 2.63 & {\bf 2.54} & 2.62 & 2.63 & 2.59 \\ \cline{2-9}
 & 90\% & 6.98 & 6.98 & 6.98 & 6.88 & 6.98 & 6.98 & {\bf 5.89} \\ \cline{2-9}
 & 95\% & 5.63 & 5.63 & 5.63 & 5.25 & 5.43 & 5.55 & {\bf 5.22} \\ \hline
 \multirow{10}{*}{Diabetes} & 50\% & 28.23 & 28.23 & 27.75 & 27.22 & 26.67 & 26.18 & {\bf 25.16} \\ \cline{2-9}
 & 55\% & 25.18 & 25.18 & 25.17 & 24.89 & 24.35 & 24.58 & {\bf 24.28} \\ \cline{2-9}
 & 60\% & 26.51 & 25.85 & 26.17 & 25.12 & 25.25 & 25.1 & {\bf 24.78} \\ \cline{2-9}
 & 65\% & 29.47 & 29.32 & 29.14 & 29.05 & 28.91 & 28.86 & {\bf 28.66} \\ \cline{2-9}
 & 70\% & 29.36 & 28.0 & 27.79 & 28.0 & {\bf 27.4} & 27.9 & 28.11 \\ \cline{2-9}
 & 75\% & 26.99 & 26.96 & 26.52 & {\bf 25.52} & 26.45 & 26.47 & 26.42 \\ \cline{2-9}
 & 80\% & 25.9 & 24.65 & 25.58 & {\bf 24.9} & 25.71 & 25.64 & 25.86 \\ \cline{2-9}
 & 85\% & 21.11 & {\bf 20.94} & 21.01 & 21.36 & 21.12 & 21.11 & 21.07 \\ \cline{2-9}
 & 90\% & 24.23 & 23.82 & 24.23 & {\bf 23.72} & 24.12 & 24.01 & 24.12 \\ \cline{2-9}
 & 95\% & 12.34 & 12.25 & 12.33 & 12.31 & 12.34 & 12.31 & {\bf 12.16} \\ \hline
\multirow{10}{*}{EEG Eye} & 50\% & 239.33 & 238.92 & 239.09 & 236.65 & 200.61 & 201.51 & {\bf 187.28} \\ \cline{2-9}
 & 55\% & 239.03 & 238.71 & 238.66 & 239.85 & 217.58 & 217.15 & {\bf 206.65} \\ \cline{2-9}
 & 60\% & 227.14 & 227.13 & 226.59 & 223.53 & 218.24 & 219.47 & {\bf 211.05} \\ \cline{2-9}
 & 65\% & 222.98 & 218.47 & 220.71 & 218.1 & 211.26 & 210.72 & {\bf 199.73} \\ \cline{2-9}
 & 70\% & 209.48 & 207.89 & 208.83 & 206.63 & 204.94 & 205.4 & {\bf 199.04} \\ \cline{2-9}
 & 75\% & 194.56 & 193.68 & 194.44 & 193.3 & 194.25 & 193.04 & {\bf 189.83} \\ \cline{2-9}
 & 80\% & 208.11 & 207.76 & 207.95 & 208.23 & 208.84 & 207.93 & {\bf 202.14} \\ \cline{2-9}
 & 85\% & 186.23 & 186.23 & 186.25 & 184.02 & 185.49 & 186.12 & {\bf 178.63} \\ \cline{2-9}
 & 90\% & 182.61 & 181.96 & 179.12 & 177.37 & 180.55 & 180.83 & {\bf 176.03} \\ \cline{2-9}
 &  95\% &  160.11 & 160.15 & 159.98 & 159.79 & 159.99 & 159.62 & {\bf 156.21} \\ \hline
\multirow{10}{*}{Australian} & 50\% & 21.88 & 21.88 & 21.87 & 21.21 & 21.76 & 20.81 & {\bf 20.38} \\ \cline{2-9}
 & 55\% &  22.7 & 22.61 & 22.7 & 22.3 & 21.62 & 21.91 & {\bf 20.2} \\ \cline{2-9}
 & 60\% & 21.23 & 21.15 & 21.0 & 21.09 & 20.58 & 21.23 & {\bf 20.05} \\ \cline{2-9}
 & 65\% & 16.82 & 16.72 & 16.65 & 15.98 & 15.94 & 15.76 & {\bf 15.07} \\ \cline{2-9}
 & 70\% & 17.47 & 17.29 & 17.46 & {\bf 16.49} & 17.24 & 17.46 & 17.43 \\ \cline{2-9}
 & 75\% & 11.02 & 11.02 & 11.02 &{\bf 10.63} & 11.27 & 11.02 & 11.02 \\ \cline{2-9}
 & 80\% & 8.28 & 8.28 & 8.09 & {\bf 8.02} & 8.06 & 8.14 & 8.17 \\ \cline{2-9}
 & 85\% & 8.01 & 7.95 & 8.01 & {\bf 7.62} & 8.08 & 8.01 & 8.01 \\ \cline{2-9}
 & 90\% & 5.79 & 5.79 & 5.79 & {\bf 5.55} & 5.92 & 5.79 & 5.78 \\ \cline{2-9}
 & 95\% & 2.88 & 2.88 & 2.88 & {\bf 2.86} & 2.92 & 2.88 & 2.88 \\ \hline
\multirow{10}{*}{Churn} & 50\% & 61.04 & 57.74 & 54.13 & 53.85 & 39.46 & 38.88 & {\bf 34.89} \\ \cline{2-9}
 & 55\% & 60.84 & 56.7 & 52.18 & 56.13 & 47.21 & 45.4 & {\bf 42.94} \\ \cline{2-9}
 & 60\% & 66.42 & 59.53 & 59.76 & 57.13 & 48.36 & 47.35 & {\bf 41.68} \\ \cline{2-9}
 & 65\% & 70.57 & 65.32 & 66.35 & 62.78 & 62.02 & 58.32 & {\bf 53.09} \\ \cline{2-9}
 & 70\% & 122.96 & 117.49 & 116.04 & 112.36 & 94.61 & 88.3 & {\bf 82.23} \\ \cline{2-9}
 & 75\% & 81.49 & 80.11 & 81.49 & 79.56 & 77.21 & 74.51 & {\bf 72.74} \\ \cline{2-9}
 & 80\% & 86.62 & 86.12 & 82.48 & 84.84 & 84.25 & 82.92 & {\bf 81.61} \\ \cline{2-9}
 & 85\% & 99.19 & {\bf 93.62} & 97.05 & 96.59 & 94.6 & 96.04 & 95.61 \\ \cline{2-9}
 & 90\% & 93.81 & 93.76 & 92.39 &{\bf 90.29} & 90.69 & 92.64 & 93.03 \\ \cline{2-9}
 & 95\% & 76.27 & 76.01 & 72.63 & 72.69 & 73.09 & 72.27 & {\bf 70.87} \\ \hline
\end{tabular}
\vspace{0.2cm}
\caption{\label{tab:linear_table_all}  {\bf  Logistic Regression Results (continued)}. }
\end{table}

 \begin{figure}[t]
  \begin{center}
    \includegraphics[width=0.8\textwidth]{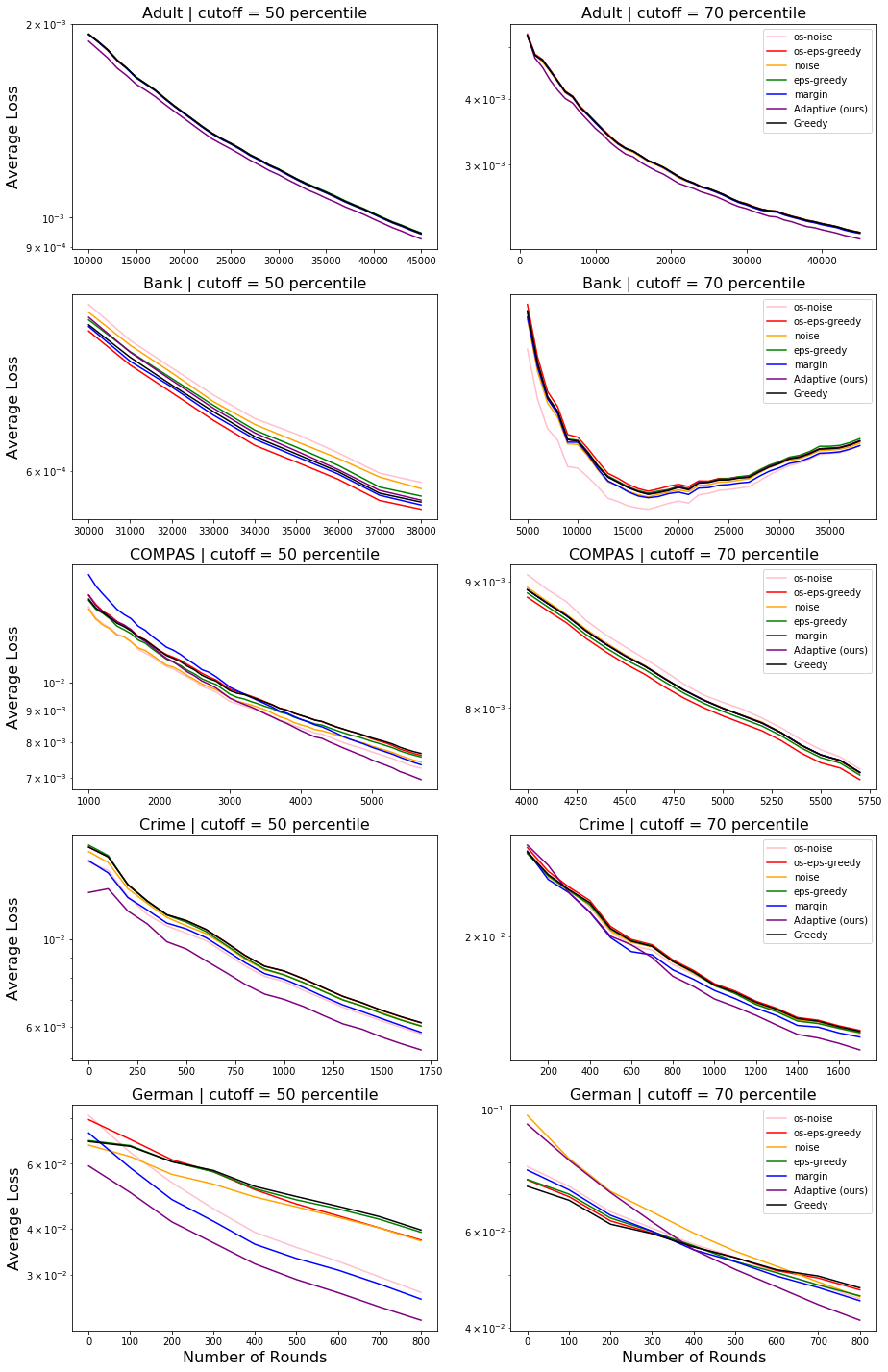}
    \end{center}
  \caption{{\bf Logistic Regression plots}. }
	\label{fig:logistic_plots_all_1}
\end{figure}

 \begin{figure}[t]
  \begin{center}
    \includegraphics[width=0.8\textwidth]{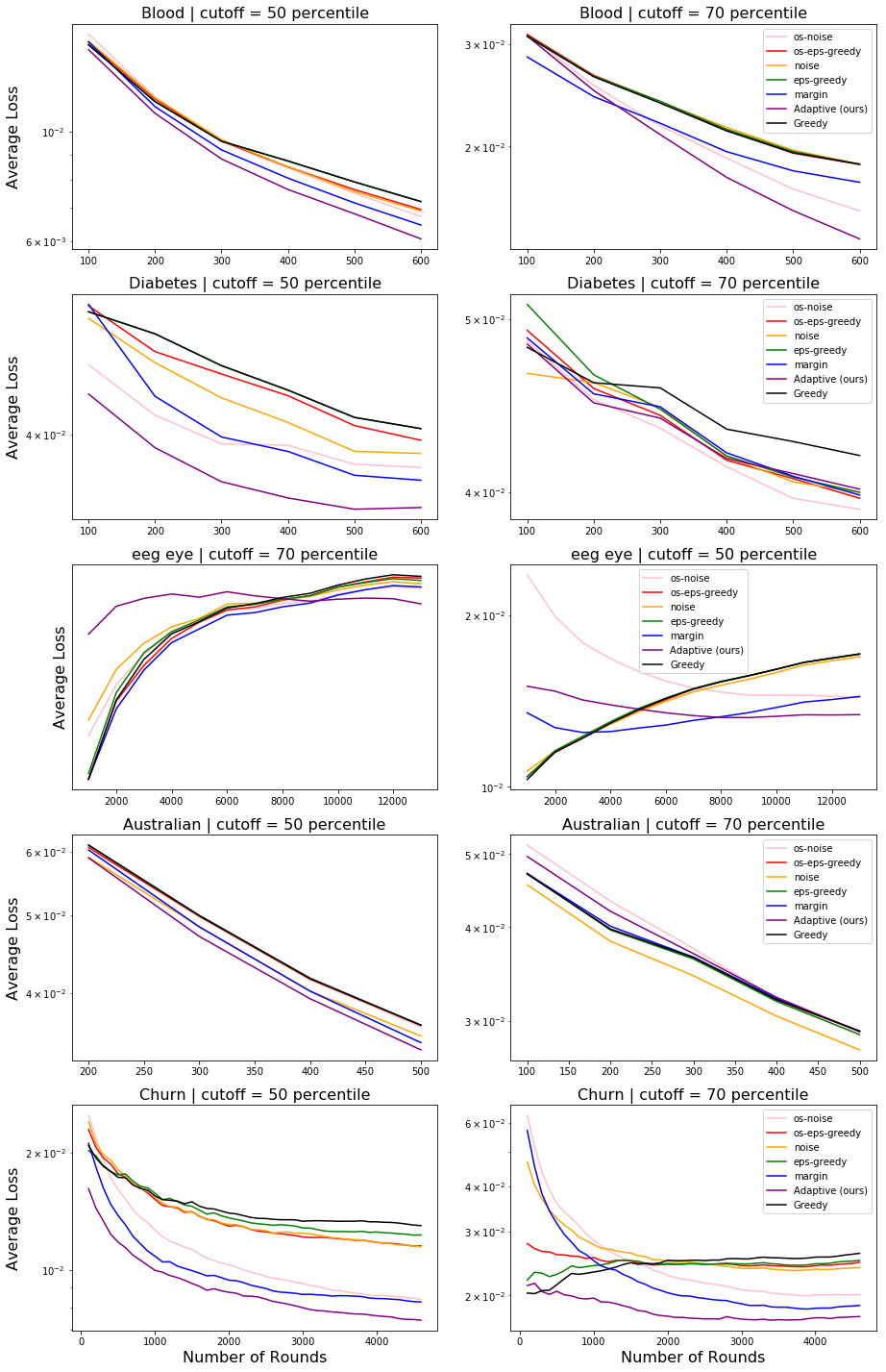}
    \end{center}
  \caption{{\bf Logistic Regression plots (continued)}. }
	\label{fig:logistic_plots_all_2}
\end{figure}

\end{document}